\newcommand{\trash}[1]{}
\newcommand{\longversion}[1]{#1}
\newcommand{\shortversion}[1]{}
\newcommand{\keepforlater}[1]{}
\shortversion{%
\documentclass[runningheads,nvcountsame,a4paper]{llncs}
}
\longversion{%
  \documentclass[10pt,letter]{article}%
  \usepackage[hscale=0.7,scale=0.75]{geometry}
  \pdfoutput=1
}

\longversion{%
  \usepackage[USenglish]{babel}%
}%

\usepackage{ifthen}
\newif\iflong

\usepackage{multirow}

\usepackage{rotating}

\usepackage{multibib}
\newcites{sec}{Appendix References}

\usepackage{xspace}
\newcommand{\etal}{et~al.\@\xspace}

\usepackage{amssymb,amsfonts,amsmath}%

\newcommand{\SBE}[2][]{\protect\ensuremath{\mathCommandFont{SE}\ifx{#1}\empty\else_{#1}\fi(#2)}}

\usepackage{cases}
\DeclareFontFamily{U}{matha}{\hyphenchar\font45}
\DeclareFontShape{U}{matha}{m}{n}{
      <5> <6> <7> <8> <9> <10> gen * matha
      <10.95> matha10 <12> <14.4> <17.28> <20.74> <24.88> matha12
      }{}
\DeclareSymbolFont{matha}{U}{matha}{m}{n}
\DeclareMathSymbol{\squplus}{2}{matha}{"5D}

\newenvironment{restatetheorem}[1][\unskip]{%
  \begingroup
  \def\thetheorem{\ref{#1}} 
}%
{%
  \addtocounter{theorem}{-1}
  \endgroup
}%

\newenvironment{restatelemma}[1][\unskip]{%
  \begingroup

}%
{%
  \addtocounter{lemma}{-1}
  \endgroup
}%

\makeatletter
\newcommand{\raisemath}[1]{\mathpalette{\raisem@th{#1}}}
\newcommand{\raisem@th}[3]{\raisebox{#1}{$#2#3$}}
\makeatother

\usepackage{stmaryrd}%

\makeatletter
\newcommand{\pushright}[1]{\ifmeasuring@#1\else\omit\hfill\ensuremath{\displaystyle#1}\fi\ignorespaces}
\newcommand{\pushleft}[1]{\ifmeasuring@#1\else\omit$\displaystyle#1$\hfill\fi\ignorespaces}
\providecommand{\leftsquigarrow}{%
	\mathrel{\mathpalette\reflect@squig\relax}%
}
\newcommand{\reflect@squig}[2]{%
	\reflectbox{$\m@th#1\rightsquigarrow$}%
}

\makeatother

\makeatletter
\newcommand{\oset}[3][0ex]{%
  \mathrel{\mathop{#3}\limits^{
      \vbox to#1{\kern-1\ex@
          \hbox{$\scriptstyle#2$}\vss}}}}
	  \makeatother

\DeclareMathOperator{\orig}{orig}%

\newcommand{\lassign}{\hspace{0.22em}\leftarrow}

\newcommand{\cp}{\text{cpy}}
\DeclareMathOperator{\type}{type}
\newcommand{\intr}{\textit{int}\xspace}
\newcommand{\labl}{\textit{label}\xspace}
\newcommand{\leaf}{\textit{leaf}\xspace}
\newcommand{\rem}{\textit{rem}\xspace}
\newcommand{\join}{\textit{join}\xspace}

\DeclareMathOperator{\Mod}{Mod}
\DeclareMathOperator{\Choose}{PGuess}

\DeclareMathOperator{\CWc}{CCon}

\DeclareMathOperator{\GS}{PCon}
\DeclareMathOperator{\GSB}{PJust}
\DeclareMathOperator{\GSA}{PPre}
\DeclareMathOperator{\GSAC}{CPre}
\DeclareMathOperator{\DS}{SProj}
\DeclareMathOperator{\AS}{AProj}

\DeclareMathOperator{\sub}{SGuess}
\DeclareMathOperator{\post}{post-order}

\newcommand{\mathCommandFont}[1]{\mathrm{#1}}

\newcommand{\complexityClassFont}[1]{\mathrm{#1}}
\newcommand{\problemFont}[1]{\textsc{#1}}

\newcommand{\hy}{\hbox{-}\nobreak\hskip0pt}
\newcommand{\set}[2]{\ifthenelse{\equal{#2}{}}{\left\{#1\right\}}{\{\,#1\mid#2\,\}}}%

\newcommand{\btheorems}[1]{\protect\ensuremath{\mathCommandFont{Th}\kern-2\mu\big(\kern-2\mu#1\kern-2\mu\big)}}
\newcommand{\Btheorems}[1]{\protect\ensuremath{\mathCommandFont{Th}\kern-2\mu\Big(\kern-2\mu#1\kern-2\mu\Big)}}
\newcommand{\Vars}[1]{\protect\ensuremath{\mathCommandFont{Vars}(#1)}}

\newcommand{\encoding}[1]{\protect\ensuremath{\left\langle #1\right\rangle}}

\def\hy{\hbox{-}\nobreak\hskip0pt}
\newcommand{\dfn}{\mathrel{\mathop:}=}

\shortversion{%
  \newcommand{\default}[2]{#1\rightarrow #2}
} %
\longversion{%
  \newcommand{\default}[2]{\frac{#1}{#2}}
}
\newcommand{\bff}{\textit{mo}\xspace}
\newcommand{\bfff}{\textit{MO}\xspace}

\newcommand{\class}[1]{\protect\ensuremath\complexityClassFont{#1}}
\newcommand{\SigmaP}[1]{{\Sigma^\class{p}_{#1}}}

\newcommand{\SAT}{\problemFont{Sat}}

\newcommand{\Ext}{\problemFont{Ext}\xspace}
\newcommand{\EXT}{\Ext}

\newcommand{\Th}[1]{\mathrm{Th}(#1)}

\newcommand{\lalpha}{\ensuremath{p}}%
\newcommand{\lbeta}{\ensuremath{j}}%
\newcommand{\lgamma}{c}%

\newlength\problemlength
\newcommand\decisionproblem[3]{%
\begin{list}{}{\labelwidth\problemlength \labelsep.7em \rightmargin1.5em
\leftmargin\problemlength \advance\leftmargin by3em
\parsep0ex \itemsep.2ex plus.1ex}
\item[{\sl Problem:\hfill}] {\problemFont{#1}}
\item[{\sl Input:  \hfill}] #2
\item[{\sl Question: \hfill}] #3
\end{list}
}

\newcommand\pproblem[4]{%
\begin{list}{}{\labelwidth\problemlength \labelsep.7em \rightmargin1.5em
\leftmargin\problemlength \advance\leftmargin by3em
\parsep0ex \itemsep.2ex plus.1ex}
\item[{\sl Problem:\hfill}] {\problemFont{#1}}
\item[{\sl Input:  \hfill}] #2
\item[{\sl Parameter:  \hfill}] #3
\item[{\sl Question: \hfill}] #4
\end{list}
}

\newcommand{\TODO}[1]{\marginpar{\small\color{red}#1}}
\newcommand{\johannes}[1]{\TODO{J: #1}}

\newcommand{\SSB}{\{\,}%
\newcommand{\SSE}{\,\}}%

\let\phi=\varphi
\let\epsilon=\varepsilon

\newcommand{\Tab}[1]{\ensuremath{\text{Child-Tabs}}}
\newcommand{\Tabs}[1]{\ensuremath{\text{Tables[#1]}}}
\renewcommand{\Big}{} 

\newcommand{\AAA}{\ensuremath{\mathcal{A}}\xspace}%
\newcommand{\BBB}{\ensuremath{\mathcal{B}}\xspace}%
\newcommand{\CCC}{\ensuremath{\mathcal{C}}}%
\newcommand{\MMM}{\ensuremath{\mathcal{M}}\xspace}%
\newcommand{\PPP}{\ensuremath{\mathcal{P}}}%
\newcommand{\TTT}{\ensuremath{\mathcal{T}}}%

\usepackage{latexsym}
\usepackage{rotating}

\usepackage{paralist}
\usepackage{enumerate}
\setdefaultleftmargin{10pt}{}{}{}{}{}

\usepackage{url}

\newcommand{\pname}[1]{\textsc{#1}\xspace}

\newcommand*\mcupinn[2]{\vcenter{\hbox{$\mathsurround=0pt
  \ifx\displaystyle#1\textstyle\else#1\fi\bigcup$}}}
\usepackage[]{xcolor}

\newcommand{\emptyfunc}{\emptyset}
\newcommand{\NAT}{\ensuremath{\mathbb{N}}}

\newcommand{\inputPredColor}{orange!55!red}
\newcommand{\outputPredColor}{blue!45!black}
\newcommand{\statePredColor}{green!62!black}

\newcommand{\comment}[1]{\begin{center}\bf*** #1 ***\end{center}}
\newcommand{\todo}[1]{{\color{red}\comment{TODO: #1}}}

\usepackage{multirow}

\usepackage{graphicx}
\graphicspath{{fig/}}
\DeclareGraphicsExtensions{.pdf,.png,.jpg}
\usepackage{booktabs}

\usepackage{microtype}
\usepackage{algpseudocode}
\algrenewcommand\algorithmicensure{\textbf{Output:}}
\algrenewcommand{\algorithmiccomment}[1]{\emph{// #1}}

\usepackage[ruled,vlined,linesnumbered]{algorithm2e}
\renewcommand*{\algorithmcfname}{Listing}
\SetKwInput{KwData}{In}
\SetKwInput{KwResult}{Out}
\SetAlFnt{\small}
\SetAlCapFnt{\small}
\SetAlCapNameFnt{\small}
\SetAlCapHSkip{0pt}
\SetEndCharOfAlgoLine{}
\IncMargin{-\parindent}

\usepackage{paralist}

\usepackage{float}
\usepackage[labelsep=quad,aboveskip=2pt]{caption}

\usepackage{subcaption}
\usepackage{csquotes}
\usepackage{paralist}
\usepackage{multicol}
\usepackage{multirow}
\usepackage{listings}

\newcommand{\tuplecolor}[1]{\textcolor{#1}}
\lstdefinelanguage{dflat}{
	numberstyle=\tiny,
	otherkeywords={:-},
	morekeywords={not},
	keywordstyle=\bfseries,emph={numChildNodes,initial,final,currentNode,childNode,bag,current,introduced,removed,atLevel,atNode,root,rootOf,leaf,leafOf,sub,childItem,childAuxItem,childCost,childOr,childAnd,childAccept,childReject,childRow},
	moreemph=[2]{item,auxItem,extend,cost,currentCost,length,or,and,accept,reject},
	alsoletter={\#}, 
	morecomment=[l]{\%},
	emphstyle=\color{\inputPredColor},
	emphstyle=[2]\color{\outputPredColor},
	literate={:-}{{$\leftarrow$}}2 {!=}{{$\neq$}}1, 
	breakindent=3em,
	escapechar=@, 
	captionpos=b,
        frame=bt, 
        numbers=left
}
 \usepackage{relsize}
\usepackage{tikz}
\usetikzlibrary{shapes,automata}
\usetikzlibrary{calc}
  \longversion{%
  \usepackage{amsthm}
  \newtheorem{theorem}{Theorem}%
  \newtheorem{lemma}{Lemma}%
  \newtheorem{proposition}{Proposition}%
  \newtheorem{example}{Example}%
  \newtheorem{definition}{Definition}%
}

  \shortversion{\renewenvironment{example}{\begin{EXa}}{\hfill\ensuremath{\blacksquare}\end{EXa}}
  \spnewtheorem{EXa}{Example}{\bfseries}{\normalfont}}
  \newtheorem{observation}{Observation}
\newcommand{\MAIRR}[2]{\ensuremath{{#1}^+_{#2}}}%
\newcommand{\BF}[1]{\ensuremath{{#1}_{\hspace{-0.3em}\bfff}}}%
\newcommand{\MAICRNCR}[2]{\ensuremath{{#1}^{?}_{#2}}}%

\newcommand{\MAIRCR}[2]{\ensuremath{{#1}^{\squplus}_{#2}}}%
\newcommand{\MARRR}[2]{\ensuremath{{#1}^-_{#2}}}%
\newcommand{\MAZR}[2]{\ensuremath{{#1}^{\sim}_{#2}}}%

\newcommand{\MARR}[2]{\ensuremath{{#1}^-_{#2}}}%
\newcommand{\MAIR}[2]{\ensuremath{{#1}^+_{#2}}}%

\newcommand{\tabval}{\ensuremath{\shortversion{\vec} u}}
\newcommand{\tabvali}[1]{\ensuremath{\vec {#1}}}
\newcommand{\tab}[1]{\ensuremath{\tau_{#1}}}
\newcommand{\at}{\protect\ensuremath{\mathCommandFont{Vars}}}

\newcommand{\att}[1]{\ensuremath{\at_{\hspace{-0.05em}\leq\hspace{-0.05em}#1}}}
\newcommand{\atto}{\ensuremath{\att{t}}}
\newcommand{\progt}[1]{\ensuremath{\prog_{\hspace{-0.05em}\leq\hspace{-0.05em}#1}}}
\newcommand{\gammat}[1]{\ensuremath{\gamma_{\hspace{-0.05em}\leq\hspace{-0.05em}#1}}}
\newcommand{\alphat}[1]{\ensuremath{\alpha_{\hspace{-0.05em}\leq\hspace{-0.05em}#1}}}
\newcommand{\betat}[1]{\ensuremath{\beta_{\hspace{-0.05em}\leq\hspace{-0.05em}#1}}}

\newcommand{\progtneq}[1]{\ensuremath{\prog_{\hspace{-0.05em}<\hspace{-0.05em}#1}}}

\newcommand{\dpa}{\ensuremath{\mathcal{DP}}}
\newcommand{\nxt}{\ensuremath{\mathcal{NSD}}}

\newcommand{\eqdef}{\ensuremath{\,\mathrel{\mathop:}=}}

\renewcommand{\P}{\text{\normalfont P}\xspace}

\newcommand{\AspComp}{\pname{CompSE}} %
\newcommand{\AspCount}{\pname{\#SE}} %
\newcommand{\AspEnum}{\pname{EnumSE}} %

\makeatletter
\newcommand{\algorithmfootnote}[2][\footnotesize]{
  \let\old@algocf@finish\@algocf@finish
  \def\@algocf@finish{\old@algocf@finish
    \leavevmode\rlap{\begin{minipage}{\linewidth}
    #1#2
    \end{minipage}}
  }
}
\makeatother

\newcommand{\PRIM}{\ensuremath{{\algo{SPRIM}}}\xspace}

\newcommand{\INCSAT}{\ensuremath{{\algo{SCONS}}}\xspace}

\newcommand\dproblem[3]{%
\begin{center}
\fbox{%
\begin{minipage}{.93	\linewidth}%
\begin{list}{}{\labelwidth\problemlength \labelsep.7em \rightmargin1.5em
\leftmargin\problemlength \advance\leftmargin by3em
\parsep0ex \itemsep.2ex plus.1ex}
\item[{\sl Problem:\hfill}] {\problemFont{#1}}
\item[{\sl Input:  \hfill}] #2
\item[{\sl Task: \hfill}] #3
\end{list}
\end{minipage}
}
\end{center}
}

\usepackage{url}

\usepackage{enumitem}

\newenvironment{indented}{\begin{changemargin}{1cm}{0cm}}{\end{changemargin}}

\let\phi\varphi
\let\epsilon\varepsilon

\renewcommand{\models}{\vDash}

\newcommand{\calT}{\mathcal{T}}

\newcommand{\Card}[1]{|#1|}
\newcommand{\CCard}[1]{\|#1\|}

\newcommand{\algo}[1]{\ensuremath{\mathsf{#1}}}

\newcommand{\NP}{\ensuremath{\textsc{NP}}\xspace}

\newcommand{\bigO}[1]{\ensuremath{{\mathcal O}(#1)}}

\newcommand{\tw}[1]{\mathit{tw}(#1)}

\newcommand{\SB}{\{}%
\newcommand{\SM}{\mid}%
\newcommand{\SE}{\}}%
\def\hy{\hbox{-}\nobreak\hskip0pt}

\newcommand{\prog}{\ensuremath{D}}

\usetikzlibrary{positioning}
\usetikzlibrary{shapes,arrows}
\tikzstyle{arg}=[draw, thick, circle]

\colorlet{afnodecolor}{green!20!blue!10}
\colorlet{tdnodecolor}{green!20!blue!10}
\colorlet{subfwnodecolor}{black!2}
\colorlet{subfwafinactivenodecolor}{white}
\colorlet{vertexTopColor}{white}
\colorlet{vertexBottomColor}{black!10}

\tikzstyle{afnode} = [draw,thick,shape=circle,minimum size=8mm,font=\normalsize,fill=afnodecolor]
\tikzstyle{afedge} = [->,draw,thick]
\tikzstyle{tdnode} = [draw,rounded corners,top color=vertexTopColor,bottom color=vertexBottomColor,minimum size=1.5em]
\tikzstyle{stdnode} = [tdnode, font=\scriptsize]
\tikzstyle{stdnodecompact} = [stdnode, inner sep = 1.5pt, outer sep = 0.1pt]
\tikzstyle{stdnodetable} = [stdnode, inner sep = 1.5pt, outer sep = 0]
\tikzstyle{stdnodenum} = [minimum size=1.5em, font=\scriptsize]
\tikzstyle{tdedge} = [-,draw,thick]
\tikzstyle{tdlabel} = [draw=none, rectangle, fill=none, inner sep=0pt, font=\scriptsize]
\tikzstyle{subfwnode} = [draw,thick,shape=rectangle,thin,rounded corners,minimum size=9mm,fill=subfwnodecolor,label distance=-2.5mm]
\tikzstyle{subfwafactivenode} = [draw,thick,shape=circle,minimum size=6mm,inner sep = 0pt,font=\scriptsize,fill=afnodecolor]
\tikzstyle{subfwafinactivenode} = [draw,thick,shape=circle,minimum size=6mm,inner sep = 0pt,font=\scriptsize,fill=white,dotted]
\tikzstyle{subfwafinactiveedge} = [->,draw,thick,dotted]

\tikzstyle{itemTree}=[level distance=2em,sibling distance=4ex,child anchor=west,grow'=right,right,align=left,every node/.style={draw,dashed,draw opacity=0.2,font=\footnotesize}]
\tikzstyle{itemTreeRoot}=[solid,inner sep=2]
\tikzstyle{orNode}=[label=left:$\lor$]
\tikzstyle{andNode}=[label=left:$\land$]
\tikzstyle{acceptNode}=[label=right:$\top$]
\tikzstyle{rejectNode}=[label=right:$\bot$]

\usepackage[]{hyperref}

\usepackage{amssymb}
\usepackage{graphicx}

\usepackage{url}
\urldef{\mailsa}\path|{fichte, hecher}@dbai.tuwien.ac.at|
\urldef{\mailsb}\path|schindler@thi.uni-hannover.de|
\newcommand{\keywords}[1]{\par\addvspace\baselineskip
\noindent\keywordname\enspace\ignorespaces#1}

\begin{document}

\shortversion{\mainmatter}  

\title{Default Logic and Bounded Treewidth%
%
%
 \thanks{The work has been supported by the Austrian Science Fund
   (FWF), Grants Y698 and P26696, and the German Science Fund (DFG), Grant ME
   4279/1-1. The first two authors are also affiliated with the
   Institute of Computer Science and Computational Science at
   University of Potsdam, Germany. \shortversion{An extended self-archived version of the paper can be found \href{https://arxiv.org/abs/1706.09393}{online}.}\longversion{%
   The final publication will be available at Springer proceedings of LATA 2018.}}%
}

\shortversion{\titlerunning{Default Logic and Bounded Treewidth}}

%
%
\author{Johannes K. Fichte$^1$%
\and Markus Hecher$^1$\and Irina Schindler$^2$\\%
\longversion{\\[3pt]
    $^1$: TU Wien, Austria, \mailsa\\
    $^2$: Leibniz Universit\"at Hannover, Germany, \mailsb}
}
\shortversion{\authorrunning{J.K. Fichte, M. Hecher, I. Schindler}}

\shortversion{\institute{Technische Universit\"at Wien, Austria\\ \and Leibniz Universit\"at Hannover, Germany\\
\mailsa\\
\mailsb\\
}}

%
%

\shortversion{\toctitle{Lecture Notes in Computer Science}
\tocauthor{Authors' Instructions}}
\maketitle


\begin{abstract}
  In this paper, we study Reiter's propositional default logic when
  the treewidth of a certain graph representation (semi-primal graph)
  of the input theory is bounded.  We establish a dynamic programming
  algorithm on tree decompositions that decides whether a theory has a
  consistent stable extension (\Ext). Our algorithm can even be used
  to enumerate all generating defaults (\AspEnum) that lead to stable
  extensions.
  %
  %
  We show that our algorithm decides \Ext in linear time in the input
  theory and triple exponential time in the treewidth (so-called
  \emph{fixed-parameter linear} algorithm).
  Further, our algorithm solves \AspEnum with a pre-computation step
  that is linear in the input theory and triple exponential in the
  treewidth followed by a linear delay to output
  solutions.  
  %
  %
  \shortversion{\keywords{Parameterized Algorithms, Tree Decompositions, Dynamic
    Programming, Reiter's Default Logic, Propositional Logic}}
\end{abstract}

\section{Introduction}

Reiter's \emph{default logic (DL)} is one of the most fundamental
formalisms to non-monotonic reasoning where reasoners draw tentative
conclusions that can be retracted based on further
evidence~\cite{Reiter80a,matr93}.  DL augments classical logic by
rules of default assumptions (\emph{default rules}). Intuitively, a
default rule expresses ``in the absence of contrary information,
assume~$\dots$''. Formally, such rule is a triple~$\default{p:j}{c}$
of formulas~$p$, $j$, and $c$ expressing ``if prerequisite~$p$ can be
deduced and justification~$j$ is never violated then assume
conclusion~$c$''. For an initial set of facts, beliefs supported by
default rules are called an extension of this set of facts.
%
%
If the default rules can be applied consistently until a fixed
point, the extension is a maximally consistent view (\emph{consistent
  stable extension}) with respect to the facts together with the
default rules.
In DL stable extensions involve the construction of the deductive
closure, which can be generated from the conclusions of the defaults
and the initial facts by means of so-called generating
defaults. However, not every generating default leads to a stable
extension. If a generating default leads to a stable extension, we
call it a stable default set.
%
%
%
%
%
Our problems of interest are deciding whether a default theory has a
consistent stable extension (\Ext), output consistent stable
default sets (\AspComp), counting the number of stable default
sets (\AspCount), and enumerating all stable default sets
(\AspEnum).
All these problems are of high worst case complexity, for example, the
problem $\Ext$ is $\SigmaP{2}$-complete~\cite{gottlob92}. 


Parameterized algorithms~\cite{CyganEtAl15} have attracted
considerable interest in recent years and allow to tackle hard
problems by directly exploiting certain structural properties present
in input instances (the \emph{parameter}). 
%
%
 For example, \EXT can be
solved in polynomial time for input theories that allow for small
backdoors into tractable fragments of
DL~\cite{FichteMeierSchindler16b}.
Another parameter is treewidth, which intuitively measures the
closeness of a graph to a tree.
\EXT can also be solved in linear time for input theories and a
(non-elementary) function that depends on the treewidth of a certain
graph representation of the default theory (incidence
graph)~\cite{msstv15}.
%
%
This result relies on logical characterization in terms of a so-called
MSO-formula and Courcelle's theorem~\cite{Courcelle90}.
Unfortunately, the non-elementary function can become extremely huge
and entirely impractical~\cite{KneisLanger09a}. More precisely, the
result by Meier et~al.~\cite{msstv15}
yields a function that is at least quintuple exponential in the
treewidth and the size of the MSO-formula.
%
%
%
%
%
%
%
%
%
%
%
This opens the question whether one can significantly improve these
runtime bounds.
A technique to obtain better worst-case runtime bounds that often even
allows to practically solve problem instances, which have small
treewidth, are dynamic programming (DP) algorithms on tree
decompositions~\cite{CharwatWoltran16a,FichteEtAl17a,FichteEtAl17b}.
%
%
In this paper, we present such a DP algorithm for DL, which uses a
slightly simpler graph notation of the theory (semi-primal graph).

\emph{Contributions.}  We introduce DP algorithms that exploit small
treewidth to solve \Ext and \AspComp in time triple exponential in the
semi-primal treewidth and linear in the input theory.  Further, we can
solve \AspCount in time triple exponential in the semi-primal
treewidth and quadratic in the input theory.
%
%
Our algorithm can even be used to enumerate all stable default
sets (\AspEnum) with a pre-computation step that is triple
exponential in the semi-primal treewidth and linear in the input
theory followed by a linear delay for outputting the solutions
(Delay-FPT~\cite{CreignouEtAl17a}).
%

\section{Default Logic}%
\label{sec:preliminaries}%
\longversion{%
  We assume familiarity with standard notions in computational
  complexity, the complexity classes~$\P$ and $\NP$ as well as the
  polynomial hierarchy. For more detailed information, we refer to
  other standard
  sources~\shortversion{\cite{Papadimitriou94}}\longversion{{\cite{Papadimitriou94,flgr06,DowneyFellows13}}}.
  For parameterized (decision) problems we refer to work by
  Cygan~\etal~\cite{CyganEtAl15}.
}
%

A \emph{literal} is a (propositional) variable
or its negation.  
%
%
%
%
The \emph{truth evaluation} of (propositional) formulas is defined in
the standard way~\cite{matr93}. In particular, $\theta(\bot) = 0$ and
$\theta(\top)=1$ for an assignment~$\theta$.
Let $f$ and $g$ be formulas and $X = \Vars{f} \cup \Vars{g}$. We write
$f \models g$ if and only if for all assignments~$\theta \in 2^X$ it
holds that if the assignment~$\theta$ satisfies $f$, then $\theta$
also satisfies~$g$.
Further, we define the \emph{deductive closure of $f$} as
$\Th{f}\dfn\set{g\in \PPP}{f\models g}$ where $\PPP$ is the family
that contains all formulas. In this paper, whenever it is clear from the context, we may use sets of 
formulas and a conjunction over formulas equivalently. In particular, we let for formula~$f$ and a family~$\MMM$
of sets of variables be
$\Mod_\MMM(f) \eqdef \{M \mid M \in \MMM, M \models f\}$.
We denote with $\SAT$ the problem that asks whether
a given formula~$f$ is satisfiable.

We define for formulas~$p$, $j$, and $c$ a \emph{default
  rule}~$d$ as a triple~$\default{p:j}{c}$; $p$ is called the
\emph{prerequisite}, $j$ is called the \textit{justification}, and $c$
is called the \textit{conclusion}; we set $\alpha(d)\dfn p$,
$\beta(d) \dfn j$, and $\gamma(d) \dfn c$. The mappings~$\alpha,\beta$ and~$\gamma$ naturally extend to sets of default rules.
We follow the definitions by Reiter~\cite{Reiter80a}. A \emph{default
theory}~$\encoding{W,D}$ consists of a set~$W$ of propositional
formulas (knowledge base) and a set of default rules.  
%

\begin{definition}\label{def:SE}
Let $\encoding{W,D}$ be a default theory and $E$ be a set of
formulas. Then, $\Gamma(E)$ is the smallest set of formulas such that:
\begin{inparaenum}[(i)]
\item $W \subseteq \Gamma(E)$
\item $\Gamma(E)=\Th{\Gamma(E)}$, and
\item for each $\default{p:j}{c}\in D$ with $p\in\Gamma(E)$ and
  $\lnot j \notin E$, it holds that $c \in\Gamma(E)$.
\end{inparaenum}
$E$ is a \emph{stable extension} of $\encoding{W,D}$, if
$E=\Gamma(E)$. An extension is \emph{inconsistent} if it contains
$\bot$, otherwise it is called \emph{consistent}.
%
%
The set
$G=\SSB d \SM \alpha(d) \in E, \lnot \beta(d)\notin E, d \in D \SSE$
is called the set of \emph{generating defaults} of extension~$E$ and
default theory~$D$.
\end{definition}

The definition of stable extensions allows inconsistent stable
extensions. However, inconsistent extensions only occur if the set~$W$
is already inconsistent where $\encoding{W,D}$ is the theory of
interest \cite[Corollary 3.60]{matr93}. In consequence, (i)~if $W$ is
consistent, then every stable extension of $\encoding{W,D}$ is
consistent, and (ii)~if $W$ is inconsistent, then $\encoding{W,D}$ has
a stable extension. For Case~(ii) the stable extension consists of all
formulas.
Therefore, we consider only consistent stable extensions.
For default theories with consistent $W,$ we can trivially transform
every formula in $W$ into a default rule. Hence, in this paper we
generally assume that $W = \emptyset$ and write a default theory
simply as set of default rules.
Moreover, we refer by $\SBE{D}$ to the set of all consistent stable
extensions of~$D$.

\begin{example}\label{ex:running1}
  Let the default theories~$D_1$ and $D_2$ be given as
  \shortversion{%
    $D_1\eqdef \{ %
    d_{1}=\default{\top : a}{a \vee b}, %
    d_{2}=\default{\top :\neg a}{\neg b}\}$ and
    $D_2\eqdef \{ %
    d_{1}=\default{c : a}{a \vee b}, %
    d_{2}=\default{c :\neg a}{\neg b},
    d_{3}=\default{\top : c}{c}, %
    d_{4}=\default{\top : \neg c}{\neg c}\}$. %
  }%
  \longversion{%
  \[D_1\eqdef \left\{ %
      d_{1}=\default{\top : a}{a \vee b},%
      d_{2}=\default{\top :\neg a}{\neg b}\right\},
  \]
  \[D_2\eqdef \left\{ %
      d_{1}=\default{c : a}{a \vee b}, %
      d_{2}=\default{c :\neg a}{\neg b},
      d_{3}=\default{\top : c}{c}, %
      d_{4}=\default{\top : \neg c}{\neg c}\right\}.\] %
}
  $D_1$ has no stable extension, while~$D_2$ has only one stable
  extension $E_{1} = \left\{ \neg c \right\}.$
\end{example}%

In our paper, we use an alternative characterization of stable
extension beyond fixed point semantics, which is inspired by Reiter's
stage construction~\cite{Reiter80a}.

\newcommand{\SD}{\text{SD}}%
\newcommand{\asat}{\lalpha\hy\text{satisfiable}\xspace}%
\newcommand{\bsat}{\lbeta\hy\text{satisfiable}\xspace}%
\newcommand{\gsat}{\lgamma\hy\text{satisfiable}\xspace}%

\begin{definition}\label{def:SED}
  Let $D$ be a default theory and $S \subseteq D$. Further, we let 
  $E(S) \eqdef \SB \gamma(d) \SM d \in S\SE$.
  We call a default~$d \in D$ \emph{\asat} in~$S$, if
  $E(S) \cup \neg\alpha(d)$ is satisfiable; and \emph{\bsat} in~$S$,
  if $E(S) \cup \beta(d)$ is unsatisfiable; \emph{\gsat} in~$S$, if
  $d\in S$.
  The set~$S$ is a \emph{satisfying default set}, if each
  default~$d \in D$ is \asat in~$S$, or \bsat in~$S$, or \gsat in~$S$.
  %
  %

  The set~$S$ is a \emph{stable default set}, if (i)~$S$ is a
  satisfying default set and
  (ii)~there is no~$S'$ where $S' \subsetneq S$ such that for each
  default~$d$ it holds that $d$ is \asat in~$S'$, or \bsat in~$S$, or
  \gsat in~$S'$.
  %
  %
  We refer by $\SD(D)$ to the set of all stable default sets of~$D$.
\end{definition}


The following lemma establishes that we can simply use stable default
sets to obtain stable extensions of a default theory.

\begin{lemma}[$\star$\footnotemark]
\label{lem:eq_se_sd}
  Let $D$ be a default theory. 
  Then, \[\SBE{D} = \bigcup_{S \in \SD(D)} \Th{\SB \gamma(d)  \SM d \in S \SE}.\]
  In particular, $S\in \SD(D)$ is a generating default of
  extension~$\Th{\SB \gamma(d) \SM d \in S \SE}$.
\end{lemma}

\footnotetext{%
  Statements or descriptions whose proofs or details are omitted due
  to space limitations are marked with ``$\star$''. These statements
  are sketched in \shortversion{an extended version}\longversion{the appendix}.
  %
  %
}

%
%




Given a default theory~$D$ we are interested in the following
problems:
\shortversion{%
  \Ext asks whether $D$ has a consistent stable extension.
 \AspComp asks to output a stable default set~$D$.
  \AspCount asks to output the number of stable default sets of~$D$.
  \AspEnum asks to enumerate all stable default sets of~$D$.
}

\longversion{%
The \emph{extension existence problem} (called $\Ext$) asks whether
$D$ has a consistent stable extension.
$\Ext$ is $\SigmaP{2}$-complete~\cite{gottlob92}.
The extension computation problem (called \AspComp) asks to output a
stable default set of~$D$.
The extension counting problem (called \AspCount) asks to output the
number of stable default sets of~$D$.
The enumerating problem asks to enumerate all stable default
sets of~$D$ (called \AspEnum).
}

\section{Dynamic Programming on TDs for Default Logic}

In this section, we present the basic methodology and definitions to
solve our problems more efficiently for default theories that have
small treewidth.  Our algorithms are inspired by earlier work for
another non-monotonic framework~\cite{FichteEtAl17a}.  However, due to
much more evolved semantics of DL, we require extensions of the
underlying concepts.

Before we provide details, we give an intuitive description.  The
property treewidth was originally introduced for graphs and is based
on the concept of a tree decomposition (TD). Given a graph, a TD
constructs a tree where each node consists of sets of vertices of the
original graph (bags) such that additional conditions hold.
Then, we define a dedicated graph representation of the default theory
and our algorithms work by dynamic programming (DP) along the tree
decomposition (post-order) where at each node of the tree, information
is gathered in tables. The size of these tables is triple exponential
in the size of the bag.
Intuitively, the TD fixes an order in which we evaluate the default
theory. Moreover, when we evaluate the default theory for one node, we
can restrict the theory to a sub-theory and parts of prerequisites,
justifications, and conclusions that depends only on the content of
the currently considered bag.

%
%
\paragraph{Tree Decompositions.}
Let $G = (V,E)$ be a graph, $T = (N,F,n)$ be a tree~$(N,F)$ with root~$n$, and
$\chi: N \to 2^V$ be a mapping.
We call the sets $\chi(\cdot)$ \emph{bags} and $N$ the set of
nodes. The pair~${\mathcal{T}} = (T,\chi)$ is a \emph{tree
  decomposition (TD)} of~$G$ if the following conditions hold:
\begin{inparaenum}[(i)]
\item for every vertex~$v \in V$ there is a node~$t \in N$ with
  $v \in \chi(t)$;
\item for every edge~$e \in E$ there is a node~$t \in N$ with
  $e \subseteq \chi(t)$; and
\item for any three nodes~$t_1,t_2,t_3\in N$, if $t_2$ lies on the
  unique path from~$t_1$ to~$t_3$, then
  $\chi(t_1)\cap \chi(t_3) \subseteq \chi(t_2)$.
\end{inparaenum}
The \emph{width} of the TD is the size of the largest bag minus one.
The \emph{treewidth}~$\tw{G}$ 
is the minimum width over all possible TDs of~$G$.
For $k\in \NAT$ we can compute a TD of width~$k$ or output that no
exists in
time~$2^{\bigO{k^3}} \cdot \Card{V}$~\cite{BodlaenderKoster08}.

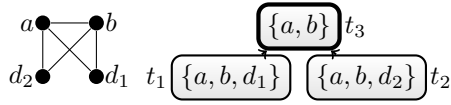
\begin{figure}[t]%
\vspace{-1.5em}
\centering
\begin{tikzpicture}[node distance=7mm,every node/.style={fill,circle,inner sep=2pt}]
\node (c) [label={[text height=1.5ex,yshift=-0.00cm,xshift=0.15cm]left:$a$}] {};
\node (d) [right = 0.5cm of c,label={[text height=1.5ex,xshift=-0.15cm]right:$b$}] {};
\node (r3) [below = 0.5cm of c,label={[text height=1.5ex,xshift=0.15cm,yshift=0.0cm]left:${d_2}$}] {};
\node (r1) [below = 0.5cm of d,label={[text height=1.5ex,xshift=-0.15cm,yshift=0.0cm]right:${d_1}$}] {};
\draw (c) to (r3);
\draw (d) to (r3);
\draw (c) to (d);
\draw (c) to (r1);
\draw (d) to (r1);
\end{tikzpicture}
\hspace{-0.5em}%
\begin{tikzpicture}[node distance=0.8mm]
\tikzset{every path/.style=thick}

\node (leaf1b) [tdnode,label={[xshift=0.25em]left:$t_1$}] {$\{a, b, d_1\}$};
\node (t7) [tdnode,label={[xshift=-0.25em]right:$t_2$}, right = 0.1cm of leaf1b,xshift=0.1cm] {$\{a,b,d_2\}$}; 
\node (join) [tdnode,ultra thick,label={[xshift=-0.25em]right:$t_3$}, above = 0.05cm of t7,xshift=-0.9cm] {$\{a, b\}$}; 

\draw [->] (t7) to (join);
\draw [->] (leaf1b) to (join);
\end{tikzpicture}
\caption{Graph~$G$ (left) and an TD~${\cal T}$ (right) of $G$.}%
\label{fig:graph-td}%
\end{figure}

Next, we restrict the TD~$\TTT$ such that we have only nice case
distinctions for our DP algorithm later. Therefore, we define a
\emph{nice TD} in the usual way as follows.
\longversion{%
  Given a TD $(T,\chi)$ with $T = (N,\cdot,\cdot)$, for
} %
\shortversion{%
  For
}%
a node~$t \in N$ we say that $\type(t)$ is $\leaf$ if $t$ has no
children; $\join$ if $t$ has children~$t'$ and $t''$ with $t'\neq t''$
and $\chi(t) = \chi(t') = \chi(t'')$; $\intr$ (``introduce'') if $t$
has a single child~$t'$, $\chi(t') \subseteq \chi(t)$ and
$|\chi(t)| = |\chi(t')| + 1$; $\rem$ (``removal'') if $t$ has a single
child~$t'$, $\chi(t) \subseteq \chi(t')$ and
$|\chi(t')| = |\chi(t)| + 1$. If every node $t\in N$ has at most two
children, $\type(t) \in \{ \leaf, \join, \intr, \rem\}$, and bags of
leaf nodes and the root are empty, then the TD is called \emph{nice}.
For every TD, we can compute a nice TD in linear time without
increasing the width~\cite{BodlaenderKoster08}.
In our algorithms we will traverse a TD bottom up, therefore, let
$\post(T,t)$ be the sequence of nodes in post-order of the induced
subtree~$T'=(N',\cdot, t)$ of $T$ rooted at~$t$.


\begin{example}
  Figure~\ref{fig:graph-td} (left) depicts a graph~$G$ together with a
  TD of width~$2$ of~$G$. 
  Further, the TD $\TTT$ in Figure~\ref{fig:running1_prim} sketches
  main parts of a nice TD of $G$ (obvious parts are left out).
\end{example}%

%
\paragraph{Graph Representations of Default Theories.} %
\longversion{%
  For a default theory~$D$, its \emph{primal graph}~$P(D)$ is the
  graph that has the variables of~$D$ as vertices and an edge~$a\, b$ if
  there exists a default~$d \in D$ and $a,b \in \at(d)$.  The
  incidence graph~$I(G)$ is the bipartite graph, where the vertices
  are variables of~$D$ and defaults~$d\in D$, and there is an
  edge~$d\,a$ between a default~$d\in D$ and a corresponding
  variable~$a\in\at(d)$.
}
\shortversion{%
  For a default theory~$D$, its \emph{incidence graph}~$I(G)$ is the
  bipartite graph, where the vertices are of variables of~$D$ and
  defaults~$d\in D$, and there is an edge~$d\,a$ between a
  default~$d\in D$ and a corresponding variable~$a\in\at(d)$.
}
The \emph{semi-primal graph}~$S(\prog)$ of $D$ is the graph, where the
vertices are variables~$\at(D)$ and defaults of~$D$. 
For each default~$d \in D$, we have an
edge~$a\, d$ if variable $a \in \at(d)$ occurs
in~$d$.  Moreover, there is an 
edge~$a\, b$ if either
$a, b\in \at(\alpha(d))$, or $a, b\in \at(\beta(d))$, or
$a, b\in \at(\gamma(d))$\footnote{Note that these formulas may also be~$\top$ or~$\bot$, which we ``simulate'' by means of the same formula~$v \vee \neg v$ or~$v \wedge \neg v$, where variable~$v$ does not occur in the default theory.}.
%
%
%
%
Observe the following connection. For any default theory~$D$, we have
that $\tw{I(D)} \leq \tw{S(D)}$.
Note that earlier work~\cite{msstv15} uses a special version of the
incidence graph~$I'(D)$. The graph $I'(D)$ is a supergraph of~$I(D)$
and still a bipartite graph, 
which contains an additional vertex for each subformula
of every occurring formula, 
and corresponding edges between subformulas and variables. 
Consequently, we obtain the bound $\tw{I(D)} \leq \tw{I'(D)}$.    

%
%

\begin{example}
  Recall default theory~$D_1$ of Example~\ref{ex:running1}. We observe
  that graph~$G$ in the left part of Figure~\ref{fig:graph-td} is the
  semi-primal graph of~$D_1$.
\end{example}

In our DP algorithms for default logic we need to remember when we can
evaluate a formula (prerequisite, justification, or conclusion) for a
default,~i.e., we have a default and all the variables of the formula
in a bag. To that end, we introduce labels of nodes.
Since we work along the TD and want a unique point where to evaluate,
we restrict a label to the first occurrence when working along the TD.
%
A \emph{labeled tree decomposition (LTD)}~$\TTT$ of a default
theory~$D$ is a tuple~$\TTT=(T,\chi,\delta)$ where $(T,\chi)$ is a TD of $S(D)$
and $\delta: N \rightarrow 2^{(\{\alpha,\beta,\gamma\} \times D)}$ is a
mapping where for any $(f,d)$ in $\{\alpha,\beta,\gamma\} \times D$ it holds that
%
(i) if $(f,d) \in \delta(t)$, then
$\{d\} \cup f(d) \subseteq \chi(t)$; and (ii) if $\{d\} \cup f(d) \subseteq \chi(t)$
and there is there is no descendent~$t'$ of $t$ such that
$(f,d) \in \delta(t')$, then $(f,d) \in \delta(t)$.

We need special case distinctions for DL. Therefore, we restrict an
LTD as follows.
For a node~$t \in N$ that has exactly one child~$t'$ where
$\chi(t) = \chi(t')$ and $\delta(t) \neq \emptyset$, we say that
$\type(t)$ is $\labl$.
If every node $t\in N$ has at most two children,
$\type(t) \in \{ \leaf,$ \join, \intr, \labl, $\rem\}$, bags of leaf
nodes and the root are empty, $\Card{\delta(t)} \leq 1$, and
$\delta(t) = \emptyset$ for $\type(t) \neq \labl$ then the LTD is
called \emph{pretty}.
It is easy to see that we can construct in linear time a pretty LTD
without increasing the width from a nice TD, simply by traversing the
tree of the TD and constructing the labels and duplicating nodes~$t$
where $\delta(t) \neq \emptyset$. Assume in the following, that we use
pretty LTDs, unless mentioned otherwise.

\begin{algorithm}[t]
  \KwData{Pretty LTD~$\TTT=(T,\chi,\delta)$ with $T=(N,\cdot,n)$ of the semi-primal graph~$S(D)$.}%
  \KwResult{A table for each node~$t\in T$ stored in a mapping
    $\Tabs{t}$.}
  \For{\text{\normalfont iterate} $t$ in \text{\normalfont post-order}(T,n)}{\vspace{-0.05em}%
    $\Tab{} \eqdef \SB \Tabs{$t'$} \SM t' \text{ is a child of $t$ in
      $T$}\SE$\;\vspace{-0.05em} %
    $\Tabs{$t$} \lassign {\PRIM}(t,\chi(t),\delta(t),{D}_t,\Tab{})$\; %
    \vspace{-0.5em} }\vspace{-0.05em}%
  \Return{\text{\normalfont \Tabs{$\cdot$}}}\vspace{-0.15em}%
  \caption{Algorithm ${\dpa}({\cal T})$ for Dynamic Programming on TD
    ${\cal T}$ for DL, cf.~\cite{FichteEtAl17a}.}
\label{fig:dpontd}
\end{algorithm}
%



Next, we briefly present the methodology and underlying ideas of our
DP algorithms on TDs. The basis for our Algorithm is given in
Listing~\ref{fig:dpontd} ($\dpa$), which traverses the underlying tree
of the given LTD~$(T,\chi,\delta)$ in post-order and runs an
algorithm~\PRIM at each node~$t\in T$.
\PRIM computes a new table~$\tab{t}$ based on the tables of the
children of~$t$.  It has only a ``local view'' on \emph{bag-defaults},
which are simply the ``visible'' defaults,~i.e.,
$\prog_t \eqdef \prog \cap \chi(t)$.
Intuitively, we store in each table information such as partial
assignments of~$D_t$, that is necessary to locally decide the default
theory without storing information beyond variables that belong to the
bag~$\chi(t)$.
Further, the \emph{default theory below $t$} is defined as
$\progt{t} \eqdef \SB d \SM d \in \prog_{t'}, t' \in \post(T,t) \SE$,
and the \emph{default theory strictly below $t$} is
$\progtneq{t}\eqdef \progt{t}\setminus \prog_t$. For root~$n$ of $T$,
it holds that $\progt{n} = \progtneq{n} = \prog$.

\begin{example}
Intuitively, the LTD of Figure~\ref{fig:graph-td} enables us to evaluate
  $\prog$ by analyzing sub-theories ($\{d_1\}$ and $\{d_2\}$) and
  combining results agreeing on $a,b$.  Indeed, for the given LTD of
  Figure~\ref{fig:graph-td}, $\progt{t_1}=\{d_1\}$,
  $\progt{t_2}=\{d_2\}$ and
  $\prog=\progt{t_3}=\progtneq{t_3}=\progt{t_1} \cup \progt{t_2}$. 
\end{example}%
%

The next section deals with the details of \PRIM.
Before, we need a notion to talk about the result of sequences of a
computation.
For a node~$t$, the Algorithm~\PRIM stores tuples in a table~$\tau_t$
based on a computation that depends on tuples (\emph{originating
  tuples}) that are stored in the table(s) of the child nodes. In
order to talk in informal explanations about properties that tuples or
parts of tuples have when looking at the entire computation from the
relevant leaves up to the node~$t$ in the post-order, we need a notion
similar to a default theory below~$t$ for parts of tuples.
Assume for now that our tuples in tables are only tuples of sets.
%
Then, we collect recursively in pre-order along the induced
subtree~$T'$ of $T$ rooted at~$t$ a sequence~$s$
of originating tuples~$(\tabval, \tabvali{\tabval_1}, \ldots, \tabvali{\tabval_m})$. If the
set~$T$ occurs in position~$i$ of tuple~$\tabval$, our
notion~$T^{\leq t}(s)$ takes the union over all
sets~$T, T_1, \ldots, T_m$ at position~$i$ in the
tuples~$\tabvali{\tabval_1}, \ldots, \tabvali{\tabval_m}$.
Since a node of type \rem will typically result in multiple
originating tuples, we have multiple sequences~$s_1,\ldots,s_m$ of
originating tuples in general. This results in a
family~$\TTT^{\leq t}\eqdef \SB T^{\leq t}(s) \SM s \in \{s_1, \ldots,
s_m\}\SE$ of such sets. However, when stating properties, we are
usually only interested in the fact that each~$S \in \TTT^{\leq t}$
satisfies the property. To this end, we refer to $T^{\leq t}$ as any
arbitrary $S \in \TTT^{\leq t}$.  Further, we let
$T^{< t} \eqdef T^{\leq t} \setminus T$.
The definition vacuously extends to nested tuples and families of
sets.
A more formal compact definition provide so-called extension
pointers~\cite{BliemCharwatHecher16b}.

\begin{example}
Recall the given TD in Figure~\ref{fig:graph-td} (right).
For illustrating notation, we remove node~$t_2$, since we only care about nodes~$t_1$ and~$t_3$
and thereby obtain a simpler TD~${\cal T}=(T,\chi,\delta)$ (of some simpler graph).
Assume that for both nodes~$t$ in $T$ we store a table of tuples, say
of the form~$\langle X, Y \rangle$,
where $X$ is a subset of the bag~$\chi(t)$
and~$Y$ is a set of subsets of~$\chi(t)$.
Further, let the tables~$\tab{i}$ for the two nodes in this example be as follows:
$\tab{1} \eqdef \{\tabvali{\tabval_{1.1}} = \langle \{d_1\}, \{\emptyset, \{d_1\}, \{d_1, b\}\} \rangle, \tabvali{\tabval_{1.2}} = \langle \{a\}, \{\{b\}\} \rangle\}$, and
$\tab{3} \eqdef \{\tabvali{\tabval_{3.1}} = \langle \emptyset, \{\{a\}\} \rangle\}$.
Then, we let tuple~$\tabvali{\tabval_{3.1}}$ originate from tuple~$\tabvali{\tabval_{1.1}}$ of child table~$\tab{1}$
and not from~$\tabvali{\tabval_{1.2}}$. We discuss only the~$Y$ part of tuple~$\tabvali{\tabval_{3.1}}$ (referred to by~$Y_{3.1}$).
In order to talk about any ``extension''~$Y^{\leq t}_{3.1.1}\supseteq Y_{3.1.1}$ of $Y_{3.1.1}=\{a\}$ in~${\cal T}$, 
we write~$Y^{\leq t}_{3.1.1}$, which can be one of~$\{a\}$, $\{a,d_1\}$, or~$\{a, d_1,b\}$.
\end{example}

\section{Computing Stable Default Sets}
\label{sec:algo:dp}
In this section, we present our table algorithm~\PRIM.
Therefore, let~$D$ be a given default theory and
$\TTT=(T,\chi,\delta)$ a pretty LTD of~$S(D)$.

Our table algorithm follows Definition~\ref{def:SED}, which consists of
two parts: (i)~finding sets of satisfying default sets of the default
theory and (ii)~generating smaller sets of conclusions for these
satisfying default sets in order to invalidate subset minimality.
Since, \PRIM has only a ``local view'' on default theory~$D$, we are
only allowed to store parts of satisfying default sets.  However, we
guarantee that, if for the ``visible'' part~$Z$ of a set of satisfying
defaults for any node~$t$ of~$T$ there is no smaller set of satisfying
defaults, then~$Z$ can be extended to a stable default set
of~$\progtneq{t}$.
However, in general $Z$ alone is not sufficient, we require auxiliary
information to decide the satisfiability of defaults.  We need a way
to prove that $Z$ witnessed a satisfying default set~$Z^{\leq t}$.  In
particular, even though each~$d \in Z^{\leq t}$ is vacuously \gsat, we
have to verify that each default~$d \in D\setminus Z^{\leq t}$ is
indeed \asat or \bsat.
In turn, we require a set~\MMM of (partial) assignments
of~$Z^{\leq t}$.
To this end, we store in table~$\tab{t}$ tuples that are of the
form~$\langle Z, {\cal M}, \PPP, \CCC \rangle$,
where~$Z\subseteq\prog_t$ and ${\cal M}\subseteq 2^{X}$ for
$X = \chi(t)\cap\at({\prog})$.
The first three tuple positions cover Part~(i) and can be seen as the
\emph{witness} part.  The last position consists of a set of
tuples~$\CCC = \langle \rho, \mathcal{AC},\mathcal{BC} \rangle$ to
handle Part~(ii) and can be seen as the \emph{counter-witness} part.

In the following, we describe more details of our tuples.
We call~$Z$ the \emph{witness set}, since $Z$ \emph{witnesses} the
existence of a satisfying default set~$Z^{\leq t}$ for a
\emph{sub-theory}~$S$.
Each element~$M$ in the set~$\mathcal{M}$ of \emph{witness models}
witnesses the existence of a model of
$F_{\leq t}\eqdef \bigwedge_{d \in Z^{\leq t}}\gamma(d) $.
For our assumed witness set~$Z$, we require a set~$\PPP$ of
\emph{witness proofs}.  The set~$\PPP$ consists of tuples of the
form~$\langle {\sigma}, \AAA, \BBB \rangle$, where
$\sigma: \prog_t \to \{\lalpha, \lbeta, \lgamma\}$ and
$\AAA, \BBB \subseteq {2^{X}}$ for $X = \chi(t)\cap\at({\prog})$.
The function~$\sigma$, which we call \emph{states function}, maps each
default~$d\in \prog_t$ to a decision
state~$v \in \{\lalpha,\lbeta,\lgamma\}$ representing the case where
$d$ is $v\hy\text{satisfiable}$.
%
%
%
The set~$\AAA$, which we call the \emph{required \lalpha\hy
  assignments}, contains an assignment $A\in 2^{X}$ for \emph{each}
default~$d$ that is claimed to be \asat. More formally, there is an
assignment~$A \in \AAA$ for each
default~$d\in\sigma^{-1}(\lalpha)\cup\progtneq{t}$ where
$\sigma^{\leq t}(d) = \lalpha$ such that there is an
assignment~$A^{\leq t}$ that satisfies
$F_{\leq t} \wedge \neg \alpha(d)$.
The set~$\BBB$, which we call the \emph{refuting \lbeta\hy
  assignments}, contains an assignment $B\in 2^{X}$ for certain
defaults.
Intuitively, for each~$B \in \BBB$ there is a default~$d$ in the
current bag~$\chi(t)$ or was in a bag below~$t$ such that there is an
assignment~$B^{\leq t}$ where the justification is not fulfilled.
More formally, there is a~$B \in \BBB$ if there is an
assignment~$B^{\leq t}$ that satisfies $F_{\leq t} \wedge \beta(d)$
for some default~$d\in \sigma^{-1}(\lbeta)\cup\progtneq{t}$ where
$\sigma^{\leq t}(d) = \lbeta$.
In the end, if $Z$ proves the existence of a satisfying default
set~$Z^{\leq t}$ of theory $\progtneq{t}$, then there is at least one
tuple~$\langle \cdot, \cdot, \BBB \rangle \in \PPP$ with
$\BBB = \emptyset$. Hence, we require that~$\BBB = \emptyset$ in order
to guarantee that each default~$d\in\progtneq{t}$ is \bsat where
$\sigma^{\leq t}(d) = \lbeta$.
To conclude, if table~$\tab{n}$ for (empty) root~$n$ contains
$\tabval=\langle Z, \cdot, \PPP, \CCC \rangle$ where~$\PPP$ contains
$\langle \cdot, \cdot, \emptyset \rangle$, then $Z^{\leq t}$ is a
satisfying default set of the default theory~$D$.
%
%
%
The main aim of $\CCC$ is to invalidate the subset-minimality of
$Z^{\leq t}$, and will be covered later.

\begin{algorithm}[t]
  \SetAlgoCaptionSeparator{$\;(\star)$:}
   \KwData{Bag $\chi_t$, label mapping~$\delta_t$, bag-theory ${D}_t$, and child tables $\Tab{}$ of $t$.}
\KwResult{Table~$\tab{t}$.} %

\lIf(\tcc*[f]{Abbreviations below.}){$\type(t) = \leaf$}{%
     $\tab{t} \lassign \Big\{ \Big\langle
     \tuplecolor{\inputPredColor}{\emptyset}, \tuplecolor{\statePredColor}{\{\emptyset\}, \{\langle \emptyset, \emptyset, \emptyset \rangle\}},~\tuplecolor{\outputPredColor}{\emptyset}
     \Big\rangle \Big\}$\label{line:leaf}%
   }%
    \uElseIf{$\type(t) = \intr, d \in D_t$ is the introduced default, and $\tau'\in \Tab{}$}{

	
	\makebox[0.0em][l]{}$\hspace{-0.3cm}\tab{t} \lassign $
	{\{%
$\Big\langle 
       	\tuplecolor{\inputPredColor}{\MAIR{Z}{d}},\,$%
       	$\tuplecolor{\statePredColor}{{\cal M}, \sub_{d,\{\lgamma\}}(\PPP)}, $
		$	%
	\tuplecolor{\outputPredColor}{\sub_{d,\{\lalpha,\lbeta,\lgamma\}}({\cal C})\, \cup \sub_{d,\{\lalpha,\lbeta\}}(\PPP, {\cal M})}\rangle,
	$}\newline
		\makebox[24.45em][l]{\makebox[1.3em]{}
		$\Big\langle 
       	\tuplecolor{\inputPredColor}{Z}, \tuplecolor{\statePredColor}{{\cal M}, \sub_{d,\{\lalpha,\lbeta\}}(\PPP)},~
	%
	\tuplecolor{\outputPredColor}{\sub_{d,\{\lalpha,\lbeta\}}({\cal C})} \rangle$
		}
       $\Bigm|\;\langle \tuplecolor{\inputPredColor}{Z}, \tuplecolor{\statePredColor}{{\cal M}, \PPP},  \tuplecolor{\outputPredColor}{{\cal C}}
       \rangle \in \tab{}'
       \Big\}$\label{line:def_int}
       }\uElseIf{$\type(t) = \labl$, $\{(\gamma,d)\}=\delta_t$ is the label of~$t$, $d\in D_t$, and $\tau'\in \Tab{}$}
       {	
%
%
	 \hspace{-0.3cm}$\tab{t}\lassign 
	\Big\{ \Big\langle \tuplecolor{\inputPredColor}{
	{Z}}, 
	\tuplecolor{\statePredColor}{\Mod_{{\cal M}}(\gamma({d})),  \GS_{d}(\PPP)},$
	%
	$\tuplecolor{\outputPredColor}{\CWc_d({\cal C})}\rangle$ 
	   \makebox[2.1em][l]{
	   } $ \mid\; 
	   \langle \tuplecolor{\inputPredColor}{Z}, \tuplecolor{\statePredColor}{{\cal M}, \PPP}, \tuplecolor{\outputPredColor}{\cal C} \rangle \in \tab{}', d\in Z\}~\cup\hspace{-5em}$\vspace{-0.05em} %
	\makebox[21.5em][l]{\makebox[1.1em]{}$\{\langle\tuplecolor{\inputPredColor}{Z}, \tuplecolor{\statePredColor}{{\cal M}, \PPP},~
	\tuplecolor{\outputPredColor}{{\cal C}}\rangle$} $\mid\;  \langle \tuplecolor{\inputPredColor}{Z}, \tuplecolor{\statePredColor}{{\cal M}, \PPP}, \tuplecolor{\outputPredColor}{\cal C} \rangle \in \tab{}', d \not\in Z\}\hspace{-5em}$\label{line:label_conc}
       }
  \uElseIf{$\type(t) = \labl$, $\{(\alpha,d)\}=\delta_t$ is the label of~$t$, $d\in D_t$, and $\tau'\in \Tab{}$}{
       



	   \makebox[25em][l]{{\hspace{-0.3cm}}$\tab{t}\lassign$
	$\{\langle\tuplecolor{\inputPredColor}{Z}$,
       	$\tuplecolor{\statePredColor}{{\cal M}, \GSA_d(\PPP, \cal M)},~ \tuplecolor{\outputPredColor}{\GSAC_d({\cal C})}\rangle$
       	} 
	   $\Bigm|\;\langle \tuplecolor{\inputPredColor}{Z}, \tuplecolor{\statePredColor}{{\cal M}, \PPP},  \tuplecolor{\outputPredColor}{{\cal C}}
       \rangle \in \tab{}'
       \Big\}$\label{line:label_pre}
	   }\uElseIf{$\type(t) = \labl$, $\{(\beta, d)\}=\delta_t$ is the label of~$t$, $d\in D_t$, and $\tau'\in \Tab{}$
	   }
	   {

		\makebox[25.0em][l]
		{$\hspace{-0.3cm}\tab{t}\lassign\Big\{ \Big\langle 
       	\tuplecolor{\inputPredColor}{Z}$, 
       	$\tuplecolor{\statePredColor}{{\cal M}, \GSB_d(\PPP, {\cal M})},~
		\tuplecolor{\outputPredColor}{
       \GSB_d({\cal C}, {\cal M})}\rangle$
       }
$\Bigm|\;\langle \tuplecolor{\inputPredColor}{Z}, \tuplecolor{\statePredColor}{{\cal M}, \PPP},  \tuplecolor{\outputPredColor}{{\cal C}}
       \rangle \in \tab{}'\Big\}\hspace{-5em}$  
       }\label{line:label_just}
       \uElseIf{
       $\type(t) = \intr, a \in \chi_t$ is the introduced variable, and $\tau'\in \Tab{}$}{%


		\makebox[25em][l]{$\hspace{-0.3cm}\tab{t} \lassign \Big\{ \Big\langle \tuplecolor{\inputPredColor}{Z}, \tuplecolor{\statePredColor}{{\cal M} \cup \MAIRCR{{\cal M}}{a}, \Choose_a(\PPP)},~$
	$\tuplecolor{\outputPredColor}{\Choose_a({\cal C})}\rangle$
       }       
       $\Bigm|\;\langle \tuplecolor{\inputPredColor}{Z}, \tuplecolor{\statePredColor}{{\cal M}, \PPP},  \tuplecolor{\outputPredColor}{{\cal C}}
       \rangle \in \tab{}'
       \Big\}$\label{line:atom_int}
    }
    \uElseIf{$\type(t) = \rem$, $d\not\in D_t$ is the removed default, and $\tau'\in \Tab{}$}{

	   \makebox[25.0em][l]{$\hspace{-0.3cm}\tab{t} \lassign \Big\{ \Big\langle \tuplecolor{\inputPredColor}{\MARR{Z}{d}}, \tuplecolor{\statePredColor}{{{\cal M}},\,} \tuplecolor{\statePredColor}{\DS_d(\PPP)},~\tuplecolor{\outputPredColor}{\DS_d({\cal C})}\Big\rangle$}
       $\Bigm|\;\langle \tuplecolor{\inputPredColor}{Z}, \tuplecolor{\statePredColor}{{\cal M}, \PPP}, \tuplecolor{\outputPredColor}{{\cal C}}
       \rangle \in \tab{}' \Big\}$\label{line:def_rem}
   
   } %
     \uElseIf{$\type(t) = \rem$, $a\not\in\chi_{t}$ is the removed variable, and $\tau'\in \Tab{}$}
     {

	   \makebox[25.0em][l]{$\hspace{-0.3cm}\tab{t} \lassign \Big\{ \Big\langle \tuplecolor{\inputPredColor}{{Z}}, \tuplecolor{\statePredColor}{\MAZR{\cal M}{a},\,} \tuplecolor{\statePredColor}{\AS_a(\PPP)},~\tuplecolor{\outputPredColor}{\AS_a({\cal C})}\Big\rangle$}
       $\Bigm|\;\langle \tuplecolor{\inputPredColor}{Z}, \tuplecolor{\statePredColor}{{\cal M}, \PPP}, \tuplecolor{\outputPredColor}{{\cal C}}
       \rangle \in \tab{}'\Big\}$\label{line:atom_rem}
   } %
     \uElseIf{$\type(t) = \join$ and $\tau', \tau'' \in \Tab{}$ with $\tau' \neq \tau''$}{%

	   \hspace{-0.3cm}{$\tab{t} \lassign \Big\{ \Big\langle \tuplecolor{\inputPredColor}{Z}, \tuplecolor{\statePredColor}{{\cal M'} \cap {\cal M''}, {\cal P'} {\hat\bowtie}_{{\cal M'}, {\cal M''}} {\cal P''}}, $} 
       $\tuplecolor{\outputPredColor}{
       ({\cal C'} {\hat\bowtie}_{{\cal M'},{\cal M''}} {\cal C''}) \cup ({\cal P'} {\hat\bowtie}_{{\cal M'},{\cal M''}} {\cal C''})\;\cup}$ 
       $\makebox[0.5cm][l]{}\tuplecolor{\outputPredColor}{({\cal C'} {\hat\bowtie}_{{\cal M'},{\cal M''}} {\cal P''})}\rangle$
	   $\makebox[4.9em][l]{}\Bigm|\;\langle \tuplecolor{\inputPredColor}{Z}, \tuplecolor{\statePredColor}{{\cal M'}, {\cal P'}}, \tuplecolor{\outputPredColor}{{\cal C}'}
       \rangle \in \tab{}', \langle \tuplecolor{\inputPredColor}{Z}, \tuplecolor{\statePredColor}{{\cal M''}, {\cal P''}}, \tuplecolor{\outputPredColor}{{\cal C}''} \rangle \in \tab{}''
       \Big\}$\label{line:join}
      
     } 
     \vspace{-0.15em}
	\Return{\tab{t}}
	\vspace{-0.35em}
     \caption{Table algorithm~$\PRIM(t,\chi_t,\delta_t,{D}_t,\Tab{})$.}
   \label{fig:prim}
   \algorithmfootnote{
\label{foot:sigma}\label{foot:abrevtwo}
  $\MARRR{S}{e} \eqdef S \setminus \{e\}$,
  $\MAZR{{S}}{e} \eqdef \{\MARRR{S}{e}\mid S\in {\cal S}\}$,
  $\MAIRR{S}{e} \eqdef S \cup \{e\}$, and
  $\MAIRCR{{\cal S}}{e} \eqdef \{\MAIR{S}{e} \mid S\in {\cal S}\}$.
}
 \end{algorithm}%

 Next, we briefly discuss important cases of Listing~\ref{fig:prim}
 for Part~(i), which consists only of the first three tuple positions
 (colored red and green) and ignores the remaining parts of the
 tuple. We call the resulting table algorithm~$\INCSAT$, which only concerns about
computing satisfying default sets.
 Let $t \in T$ and $\tabval' = \langle Z, {\cal M}, \PPP, \cdot \rangle$
 a tuple of table~$\tau'$ for a child node of~$t$ and
 $\langle \sigma, \AAA, \BBB \rangle$ a tuple in~$\PPP$.
 We describe informally how we transform $\tabval'$ tuples into one or
 more tuples for the table in node~$t$.

 If $t$ is of type~$\intr$ and a default~$d$ is introduced in~$t$,
 Line~\ref{line:def_int} guesses whether~$d$ is \asat, \bsat, or
 \gsat. To this end, $\sub_{d,{\cal S}}(\PPP)$ adds potential proofs
 to~$\PPP$ where the satisfiability state of~$d$ is within~${\cal S}$.
 Lines~\ref{line:label_conc},~\ref{line:label_pre}
 and~\ref{line:label_just} cover nodes of type $\labl$ 
 as follows: In Line~\ref{line:label_conc}, if~$(\gamma,d)$ is the
 label and $\sigma(d) = \lgamma$, we enforce that each~$M\in{\cal M}$
 is also a model of~$\gamma(d)$. $\GS_d(\PPP)$ only keeps tuples in
 $\PPP$ where each~$A \in \AAA$ is a model of~$\gamma(d)$.
 In Line~\ref{line:label_pre}, if~$(\alpha,d)$ is the label and
 $\sigma(d) = \lalpha$, $\GSA_d(\PPP, {\cal M})$ enforces that
 each~$A \in\AAA$ within~$\PPP$ is a model of $\neg \alpha(d)$.
 In Line~\ref{line:label_just}, if~$(\beta,d)$ is the label and
 $\sigma(d) = \lbeta$, $\GSB_d(\PPP, {\cal M})$ adds assignments
 of~$\MMM$ to~$\BBB$ that are also models of~$\beta(d)$.

 Next, we cover the case, where a variable~$a$ is introduced.  In
 Line~\ref{line:atom_int}, we increase the existing witness
 set~$M \cup \{a\}$ for each $M \in \MMM$. $\Choose_a(\PPP)$ works
 analogously for $\BBB$ and computes all potential combinations of
 every~$A \in \AAA$, where $a$ is either set to true or to false.

 In Line~\ref{line:def_rem}, we remove default~$d$ from~$Z$ and
 $\DS_d(\PPP)$ removes $d$ from the domain of the mapping~$\sigma$,
 since $d$ is not considered anymore.
 In Line~\ref{line:atom_rem}, we remove variable~$a$ from each
 $M \in \MMM$ and $\AS_a(\PPP)$ works analogously for each assignment of~$\AAA$ and
 $\BBB$.

 Finally, if the node is of type~$\join$, we have a second child and
 its table~$\tab{}''$ as well as a tuple~$\tabval'' \in \tab{}''$.
 Intuitively, tuples~$\tabval'$ and $\tabval''$ represent intermediate
 results of two different branches in~$T$. To combine these results,
 we have to join the tuples on the witness extension, witness states,
 and the witness models. The join operation~$\bowtie$ can be seen as a
 combination of inner and outer joins, used in database
 theory~\cite{AbiteboulHullVianu95}.  Note that for instance for an
 assignment~$B \in \BBB$ to endure within~$\PPP$ of~$\tab{t}$, it
 suffices that $B$ is a corresponding witness model in~$\tabval{}''$.

\begin{example}\label{ex:sat}
  Consider default theory~$\prog$ from Example~\ref{ex:running1} and
  in Figure~\ref{fig:running1_prim} (left) pretty
  LTD~$\TTT=(\cdot, \chi, \delta)$ of the semi-primal graph~$S(\prog)$ and the
  tables~$\tab{1}$, $\ldots$ , $\tab{18}$ illustrating computation
  results obtained during post-order traversal of ${\cal T}$ by $\dpa$
  using \INCSAT instead of \PRIM in Line~3. We omit the last position
  of the tuples, since those are only relevant for \PRIM.
  Note that we discuss only selected cases, and we assume for presentation that each tuple in a table $\tab{t}$
is identified by a number,~i.e., the $i$-th tuple corresponds to $\tabvali{\tabval_{t.i}} =
\langle Z_{t.i}, {\cal M}_{t.i}, {\cal P}_{t.i}, {\cal C}_{t.i} \rangle$. The numbering
naturally extends to sets in witness proofs and counter-witnesses.
  %
  %


  We obtain
  table~$\tab{1}=\SB \langle\emptyset, \{\emptyset\}, \{\langle
  \emptyset, \emptyset, \emptyset \rangle \}\rangle \SE$ as
  $\type(t_1) = \leaf$ (see Line~\ref{line:leaf}).
  Since $\type(t_2) = \intr$ and~$a$ is the introduced variable, we
  construct table~$\tab{2}$ from~$\tab{1}$ by modifying~${\cal M}_{2.1}$
  and
  ${\cal P}_{2.1}=\{\langle \sigma_{1.1}, \AAA_{1.1}, {\cal M}_{2.1}\rangle\}$,
  where~${\cal M}_{2.1}$ contains~$M_{1.1.k}$ and
  $M_{1.1.k}\cup \{a\}$ for each~$M_{1.1.k}$ ($k\leq 1$) in
  $\tab{1}$. This corresponds to a guess on~$a$.  Precisely,
  ${\cal M}_{2.1}\eqdef\{\emptyset, \{a\}\}$
  (Line~\ref{line:atom_int}).

  Then, $t_3$ introduces default $d_1$, which results in two tuples.  In
  tuple~$\tabvali{\tabval_{3.1}}$ default~$d_1$ is \asat or \bsat due
  to~$\alpha(d_1)$ or~$\beta(d_1)$ (see~$\PPP_{3.1}$,
  Line~\ref{line:def_int}).
  In tuple~$\tabvali{\tabval_{3.2}}$ default~$d_1$ is \gsat and we have that
  $Z_{3.2}=\{d_1\}$ and
  $\PPP_{3.2}=\SB\{d_1\mapsto\lgamma\}, \emptyset,
  \emptyset\rangle\SE$.

  Node~$t_4$ introduces label~$(\beta,d_1)$ and
  modifies~$\PPP_{4.1.2}$. In particular, it chooses among~$\MMM$
  candidates, which might contradict that~$d_1$ is \bsat (see
  Line~\ref{line:label_just}).  Obviously, we have that
  $\BBB_{4.1.2} = \SB\{a\}\SE$, since $\beta(d_1)=a$.

  In table~$\tab{5}$, we present the case where default~$d_1$ should be
  \asat.  In this case since $\alpha(d_1)=\top$, we do not find any
  model of~$\bot$. In consequence, there is no corresponding 
successor of~${\cal P}_{4.1.1}$ in~$\tab{5}$, i.e., in~$\tab{5}$ it turns out that $d_1$ can not be~\asat.

  Table~$\tab{7}$ concerns the conclusion~$\gamma(d_1)$ of a
  default. It updates every assignment occurring in the table, such
  that the models satisfy~$\gamma(d_1)$ if $d_1$ is \gsat.  The
  remaining cases work similarly.

  In the end, join node~$t_{16}$ just combines witnesses agreeing on
  its content.
\end{example}%

\begin{figure}[t]%
\centering %
\includegraphics[scale=1.1]{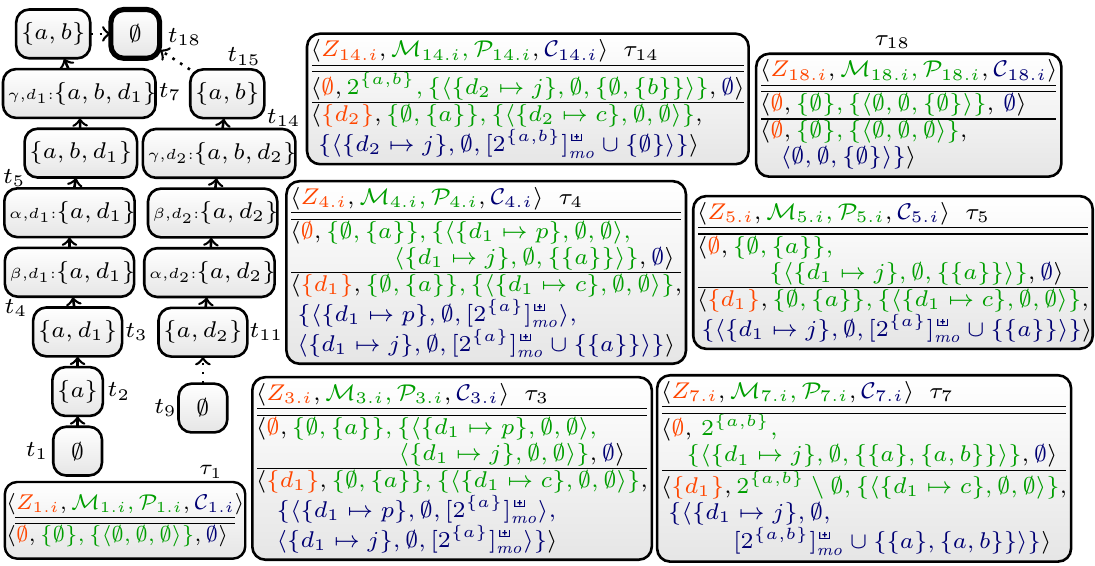}
\caption{Selected DP tables of~\PRIM~for pretty LTD~$\TTT$.}
\label{fig:running1_prim}
\label{fig:running1_prim_asp}
\end{figure}%

Next, we briefly discuss the handling of counter-witnesses, which
completes Algorithm~\PRIM. The handling of
counter-witnesses~${\cal C}$ is quite similar to the witness
proofs~$\PPP$.
The tuples~$\langle {\rho}, {\cal AC}, {\cal BC} \rangle \in{\cal C}$
consist of a states function~$\rho: \progt{t} \mapsto \{\lalpha,\lbeta,\lgamma\}$,
required \lalpha\hy assignments~${\cal AC}\subseteq 2^{X}$
and refuting \lbeta\hy assignments ${\cal BC}\subseteq 2^{(X\cup \{\bff\})}$ for
$X = \atto\cap\chi(t)$.  
In contrast to the refuting \lbeta\hy assignments in $\BBB$,
${\cal BC}$ may in addition contain an assignment~$B \in \mathcal{BC}$
with a marker~\bff. The marker indicates that $B^{\leq t}$ is actually
not refuting, but only a model of $\gamma(d)$ for each default below
$t$ that is \gsat, i.e.,
$\bigwedge_{d \in \progt{t}, \rho^{\leq t}(d) = \lgamma} \gamma(d)$.
In other words, those assignments setting~\bff to true are the
counter-witness assignments that do not refute \lgamma\hy assignments
(comparable to witness assignments in \MMM for Part~(ii)).

The existence of a certain counter-witness tuple for a witness in a
table~$\tab{t}$ establishes that the corresponding witness can \emph{not} be
extended to a stable default set of~$\progt{t}$.  In particular, there
exists a stable extension for~$\prog$ if the table~$\tab{n}$ for
root~$n$ contains a tuple of the
form~$\langle\emptyfunc, \{\emptyset\}, \PPP, {\cal C} \rangle$,
where~$\PPP\neq\emptyset$ and contains tuples of the
form~$\langle \cdot, \cdot, \emptyset\rangle$.
Moreover, for
\emph{each}~$\langle \rho, {\cal AC}, {\cal BC} \rangle \in {\cal C}$ there
is $\emptyset \in \mathcal{BC}$ indicating a true refuting \lbeta\hy
assignment for the empty root~$n$.  Intuitively, this establishes that
there is no actual counter-witness, which contradicts that the
corresponding satisfying default~$Z^{\leq t}$ is subset-minimal and
hence indeed a stable default set.

%
%
%
Due to space limitations, we omit a full description of both Parts~(i)
and (ii) together for our algorithm. A major difference of Part~(ii)
is that we need a special function~$\CWc_d({\cal C})$ to establish
that a default~$d$ is \bsat, which is defined with respect to fixed
set~$S$,~c.f., Case~(ii) of Definition~\ref{def:SED}.
Then, $\CWc_d({\cal C})$ additionally adds potential proofs involving
counter-witnesses and \bff models, where $\rho(d) \neq \lgamma$,
but~$\sigma(d) = \lgamma$.

In the following, we state the correctness of the algorithm ${\dpa}$.

\begin{theorem}[$\star$]\label{thm:prim:correctness}
  Given a default theory~$\prog$, algorithm $\dpa$ correctly
  solves \Ext.
\end{theorem}
\begin{proof}[Idea]
  The correctness proof of this algorithm needs to investigate each
  node type separately.  We have to show that a tuple at a node~$t$
  guarantees existence of a stable default set for a sub-theory of theory~$\prog_{\leq t}$, 
  which proves soundness. Conversely, one
  can show that each stable default set is indeed evaluated while
  traversing the pretty LTD, which establishes completeness.
\end{proof}


Next, we establish that we can extend $\dpa$ to enumerate
stable default sets. The algorithm on top of $\dpa$ is
relatively straight forward, which can be found in \shortversion{an extended version}\longversion{the appendix}.
The idea is to compute a first stable default set in linear time,
followed by systematically enumerating subsequent solutions with linear delay.
One can even further extend~$\dpa$ to solve~\AspCount, similar to related work~\cite{FichteEtAl17a} in a
slightly different context.

\begin{theorem}[$\star$]\label{thm:enumsafe}
  Given a default theory~$\prog$, algorithm~$\PRIM$ can be used as
  a preprocessing step to construct tables from which we can
  solve problem~\AspEnum.
\end{theorem}
\begin{proof}[Idea]
The correctness proof requires to extend the previous
results to establish a one-to-one correspondence when traversing the
tree of the TD and such that we can reconstruct each solution as well as we do not
get duplicates. The proof proceeds similar to Theorem~\ref{thm:prim:correctness}.
\end{proof}

The following theorem states that we obtain threefold exponential
runtime in the treewidth.

\begin{theorem}[$\star$]\label{thm:prim:runtime}
  Algorithm $\dpa$ runs in
  time \hbox{$\bigO{{2^{{2^{2^{k + 4}}}}}\cdot \CCard{S(\prog)}}$}
  for a given default theory~$\prog$, where $k\eqdef\tw{S(\prog)}$ is the treewidth of semi-primal
  graph~$S(\prog)$.
\end{theorem}





\section{Conclusion}\label{sec:conclusions}
In this paper, we established algorithms that
operate on tree decompositions of the semi-primal graph of a given
default theory. Our algorithms can be used to decide whether the
default theory has a stable extension or to enumerate all
stable default sets. The algorithms assume small treewidth and run in linear time and
with linear delay, respectively. Even though already linear time results 
for checking the existence of a stable extension are known, 
we are able to establish runtime that is only triple exponential in the 
treewidth of the semi-primal graph. 

In order to simplify the presentation, we mainly
covered the semi-primal graph. However, we believe that our algorithms can be extended to tree decompositions
of the incidence graph. Then we need additional states to
handle the cases where prerequisite, justification, and conclusion do not occur
together in one bag. Consequently, such an algorithm will likely be
very complex. Further, we also believe that our algorithm can be extended to disjunctive defaults~\cite{GeLiPrTr1991}, 
where we have to guess which of the conclusion parts is to 
apply.
An interesting research question is whether we can improve our runtime
bounds. 
%
Still it might be worth implementing our algorithms to
enumerate stable default sets for DL, as previous work
showed that a relatively bad worst-case runtime may anyways lead to
practical useful results~\cite{CharwatWoltran16a}.



\shortversion{\bibliographystyle{splncs}}
\longversion{\bibliographystyle{abbrv}}
\bibliography{references}

\longversion{
\newpage
\appendix

\section{Equivalence of stable default sets and stable extensions}
We provide in the following insights on the correspondence between sets
$\SD(D)$ and~$\SBE{D}$ for a given default theory~$D$.

\begin{restatelemma}[lem:eq_se_sd]
\begin{lemma}
  Let $D$ be a default theory. 
  Then, \[\SBE{D} = \bigcup_{S \in \SD(D)} \Th{\SB \gamma(d)  \SM d \in S \SE}.\]
  In particular, $S\in \SD(D)$ is a generating default of
  extension~$\Th{\SB \gamma(d) \SM d \in S \SE}$.
\end{lemma}
\end{restatelemma}
\begin{proof}[Sketch]
Consider an arbitrary default theory~$D$.\\
``$\Longrightarrow$'': Take any extension~$E\in\SBE{D}$. Observe that~$E$ is closed unter~$\Th{\cdot}$, i.e., $E=\Th{E}$. From~$E$, we now construct set~$S\eqdef \{d\in D \mid \gamma(d)\in E, \alpha(d) \in E, \neg\beta(d)\not\in E\}$.
Assume towards a contradiction that~$S$ is not a stable default set, so it either dissatisfies at least one default, which immediately leads to a contradiction
since~$E$ is a stable extension, or~$S$ is not subset-minimal. If~$S$ is not subset-minimal, there is a smaller set~$S'\subsetneq S$,
which is a satisfying default set. Observe that there is at least one~$d\in S\setminus S'$ where~$\gamma(d)\not\in \Th{\SB \gamma(d')  \SM d' \in S' \SE}$,
since otherwise~$S'$ can not be a satisfying default set due to $\Th{\SB \gamma(d) \SM d \in S \SE} = \Th{\SB \gamma(d')  \SM d' \in S' \SE}$ and $S'\subsetneq S$, which results in at least one default in $S\setminus S'$ that is dissatisfied (by construction of~$S$, c.f.\ Definition~\ref{def:SED}(i)). As a result, there is a smaller extension~$E'\subsetneq E$, where~$E' \eqdef \Th{\SB \gamma(d')  \SM d' \in S' \SE}$, which contradicts, once again, that~$E$ is a stable extension.\\
``$\Longleftarrow$'': Assume any stable default set~$S$. We define~$E \eqdef \Th{\SB \gamma(d)  \SM d \in S \SE}$.
Assume towards a contradiction that~$E$ is not stable. Obviously, by construction of~$E$, $\Gamma(E)\eqdef E$
satisfies Definition~\ref{def:SE}(i)-(iii). It remains to show, that, indeed there is no smaller~$\Gamma'(E)\subsetneq\Gamma(E)$
which also satisfies the three conditions. Assume towards a contradiction, that such a set~$\Gamma'(E)$ with $\Gamma'(E) = \Th{ \Gamma'(E)}$ indeed exists.
Then there is at least one default~$d\in D$, such that~$\gamma(d)\in\Gamma(E)\setminus\Gamma'(E)$. As a result, by construction of~$E$, $S$ can not be stable default set, which yields a contradiction.

\end{proof}

\section{Auxiliary Definitions of Table algorithm~\PRIM}

We provide formal definitions for abbreviations that are used in
algorithm~\PRIM, which we explained only verbally in
Section~\ref{sec:algo:dp}. We abbreviate by~$\BF{\mathcal{S}\,}$ the
set $\{{S} \mid {S}\in {\mathcal{S}}, \bff\in {S}\}$ for
set~${\cal S}$ of sets.  Further, we define the
abbreviation~$\MAICRNCR{{\cal S}}{e}$ for set~${\cal S}$ of sets as
follows: $\MAICRNCR{\emptyset}{e} \eqdef \{\emptyset\}$ and
$\MAICRNCR{{\cal S}}{e} \eqdef \bigcup_{S\in{\cal S}, S' \in
  \MAICRNCR{({\cal S}\setminus S)}{e}} \{S' \cup \{\MAIR{S}{e}\}, S'
\cup \{S\}\}$.

\begin{align}
&\cp_d({\cal P}, \pi) &&\eqdef \{\langle \sigma, {\cal A}, {\cal B}\rangle \mid \langle \sigma, {\cal A}, {\cal B} \rangle \in {\cal P}, \sigma(d) \neq \pi\}\\
&\sub_{d,{\cal S}}({\cal P}, {\cal M}) &&\eqdef \{
	\langle \MAIR{\sigma}{d\mapsto\pi}, {\cal A}, \MAIR{\cal M}{\bff}\cup{\cal B} \rangle
	\mid \langle \sigma, {\cal A}, {\cal B}\rangle \in {\cal P}, \pi\in{\cal S}\}\\
&\sub_{d,{\cal S}}(\PPP) &&\eqdef \sub_{d,{\cal S}}(\PPP, \emptyset)\\
&\GS_d({\cal P}) &&\eqdef \{\langle \sigma, {\cal A}, \Mod_{\cal B}(\gamma(d)) \rangle \mid \langle \sigma, {\cal A}, {\cal B} \rangle \in {\cal P}, \sigma(d) = \lgamma, \\
& &&\phantom{{}\eqdef}{\cal A} =\Mod_{\cal A}(\gamma(d))\}\notag\\
&\CWc_d({\cal C}) &&\eqdef \GS_{d}({\cal C}) \cup \{\langle \rho, {\cal AC}, \BF{\cal BC\,} \cup \Mod_{\cal BC}(\gamma(d)) \rangle\\
& &&\phantom{{}\eqdef}\mid \langle \rho, {\cal AC}, {\cal BC} \rangle \in {\cal C}, \rho(d)\neq \lgamma\}\notag\\
&\GSA_d({\cal P}, {\cal M}) &&\eqdef \cp_d({\cal P}, \lalpha) \cup \{\langle \sigma, {\cal A} \cup {\cal A'}, \BBB \rangle \mid {\langle \sigma, {\cal A}, \BBB \rangle \in {\cal P},}\\
& &&\phantom{{}\eqdef}\sigma(d)=\lalpha, {\cal A'} \in \Mod_{{\cal M} \cup \MAZR{{\cal B}}{\bff}}({\neg{\alpha(d)}})\}\notag\\
&\GSAC_d	({\cal C}) &&\eqdef \GSA_d({\cal C}, \emptyset)\\
&\GSB_d(\PPP, {\cal M}) &&\eqdef \cp_d(\PPP, \lbeta) \cup \{\langle \sigma, \AAA, \BBB \cup \MAZR{[{\Mod_{\cal M}(\beta(d))}]}{\bff}) \rangle\\
& &&\phantom{{}\eqdef}\mid\langle \sigma, \AAA, \BBB \rangle \in \PPP, \sigma(d) = \lbeta \}\notag\\
&\Choose_a(\PPP) &&\eqdef \{\langle \sigma, \AAA', \BBB \cup \MAIRCR{\BBB}{a} \rangle \mid\AAA' \in \MAICRNCR\AAA{a}, \langle \sigma, \AAA, \BBB \rangle \in \PPP\}\\
&\DS_d(\PPP) &&\eqdef \{\langle \sigma \setminus \{d\mapsto \lalpha, d\mapsto\lbeta, d\mapsto\lgamma\}, \AAA, \BBB \rangle\hspace{-0.1em}\mid\hspace{-0.1em} \langle \sigma, \AAA, \BBB \rangle \in \PPP\}\hspace{-5em}\\
&\AS_a(\PPP) &&\eqdef \{\langle \sigma, \MAZR{\AAA}{a}, \MAZR\BBB{a} \rangle \mid \langle \sigma, \AAA, \BBB \rangle \in \PPP\}\\
&{\cal M'}\bowtie{\cal M''} &&\eqdef \{ M' \cup M'' \mid M' \in {\cal M'}, M'' \in {\cal M''}, M' \cap \MAIR{[\chi_t]}{\bff} =\\
& &&\phantom{{}\eqdef} M'' \cap \MAIR{[\chi_t]}{\bff} \}\notag\\
&{\cal B'}\bowtie_{{\cal M'},{\cal M''}}{\cal B''} &&\eqdef [{\cal B'} \bowtie ({\cal B''} \cup {{\cal M''}})] \cup [({\cal B'} \cup {{\cal M'}}) \bowtie {\cal B''}]\\
&{\cal P'} {\hat{\bowtie}}_{{\cal M'}, {\cal M''}} {\cal P''} &&\eqdef \{\langle \sigma, {\cal AR}, {\cal B'} \bowtie_{{\cal M'}, {\cal M''}} {\cal B''}\rangle \mid \langle \sigma, \AAA', {\cal B'} \rangle \in {\cal P'},\\
& &&\phantom{{}\eqdef}\langle \sigma, \AAA'', {\cal B''} \rangle \in {\cal P''}, {\cal AR} = {\cal A'} {{\bowtie}}(\AAA''\cup{{\cal M}'' \cup \MAZR{[{\cal B''}]}{\bff}}),\notag\\
& &&\phantom{{}\eqdef}\AAA' \cup \AAA'' \subseteq {\cal AR} \}\,\cup
\{\langle \sigma, {\cal RA}, {\cal B'} \bowtie_{{\cal M'}, {\cal M''}} {\cal B''}\rangle\notag\\
& &&\phantom{{}\eqdef}\mid \langle \sigma, \AAA', {\cal B'} \rangle \in {\cal P'}, \langle \sigma, \AAA'', {\cal B''} \rangle \in {\cal P''},\notag\\
& &&\phantom{{}\eqdef}{\cal RA} = {\cal A''} {{\bowtie}} (\AAA'\cup{{\cal M}' \cup \MAZR{[{\cal B'}]}{\bff}}), \AAA' \cup \AAA'' \subseteq {\cal AR} \}
\end{align}

\section{Proof of Correctness}

Before we provide more insights on the correctness of our algorithms,
we require some missing auxiliary definitions.

\paragraph{Bag-default parts.}
Consider an LTD~$(T,\chi,\delta)$ of the graph~$S(\prog)$ of a given default theory~$\prog$. The
set~$\atto \eqdef \SB v \SM v \in \at(\prog) \cap \chi(t'), t' \in
\post(T,t) \SE$ is called \emph{variables below~$t$}.
Further, the \emph{bag-default parts} for prerequisite, justification, or
conclusion $f \in \{\alpha, \beta, \gamma\}$ contain
$f_t\eqdef\SB f(d) \SM (f,d) \in \delta(t) \SE$.
We naturally extend the definition of the bag-default parts to the respective default
parts below~$t$ (analogously to our definitions for default theory below~$t$), i.e., 
we also use~$\alphat{t}, \betat{t}$, and~$\gammat{t}$.

Further, we define mapping $\Gamma_t: 2^{\gamma(\progt{t})} \rightarrow 2^{\gammat{t}}$ by  $\Gamma_t[E]\eqdef E \cap \gammat{t}$.
\begin{example}
Recall the LTD of Figure~\ref{fig:running1_prim}.
Observe that $\att{t_6}=\at(\prog)$, $\alphat{t_6}=\{\alpha(d_1)\}$ and $\gammat{t_{16}}=\{\gamma(d_1),\gamma(d_2)\}$.
\end{example}%

We employ the correctness argument using the notions of (i)~\emph{partial solutions}
consisting of \emph{partial extensions} and the notion of
(ii)~\emph{local partial solutions}.

\begin{definition}\label{def:partial-model}
  Let $\prog$ be a default theory, $\calT = (T, \chi, \delta)$ be an LTD of
  the semi-primal graph~$S(\prog)$ of~$\prog$, where
  $T=(N,\cdot,\cdot)$, and $t\in N$ be a node.  Further, let
  $\emptyset \subsetneq \mathcal{B} \subseteq 2^{\atto \cup
    \{\bff\}}$, $\AAA \subseteq 2^{\MAZR{{\cal B}\,}{\bff}}$,
  ${\sigma}: \progt{t} \to \{\lalpha, \lbeta,
  \lgamma\}$, $E\supseteq \gamma(Z)$, where
  $Z\eqdef \sigma^{-1}(\lgamma)$. The
  tuple $({\sigma}, \AAA, \BBB)$ is a \emph{partial extension
   under~$E$ for~$t$} 
if the following conditions hold:
  \begin{enumerate}
  \item ${Z}$ is a set of satisfying defaults of
    $\progtneq{t}\setminus [\{d\in \progtneq{t} \mid \sigma(d) =
    \lbeta, \exists {B}\in {\mathcal{B}}: {B}\models \Gamma_t[E]
    \wedge \beta(d)\}]$, 
  \item $\AAA$ is a set such that:
    \begin{enumerate}
    \item $\Card{\AAA} \leq \Card{\sigma^{-1}(\lalpha)} - 1$,
    \item
      $\exists d\in \progt{t}: \sigma(d) = \lalpha,
      \alpha(d)\in\alphat{t}, A \models \Gamma_t[\gamma(Z)] \wedge
      \neg\alpha(d)$ \quad for every
      $A\in \AAA$, 
    \item $\exists A\in \AAA: A\models \Gamma_t[\gamma({Z})] \wedge \neg
      \alpha(d) \Longleftarrow
      {\sigma}(d)=\lalpha$ \qquad\quad for
      every~$d\in\progt{t}$ such that\ $\alpha(d)\in\alphat{t}$; and
    \end{enumerate}
  \item ${\cal B}$ is the largest set such that:
    \begin{enumerate}
    \item $B \models \Gamma_t[\gamma(Z)]$ \qquad\qquad\qquad\qquad\qquad\qquad\qquad\qquad\qquad\, for every $B\in{\cal B}$,
    \item $\exists d\in \progt{t}: \sigma(d) = \lbeta,
      \beta(d)\in\betat{t}, B \models \Gamma_t[E] \wedge
      \beta(d)$ \quad for every $B\in{\cal B}$ where $\bff\not\in B$.
    \end{enumerate}
\end{enumerate}
\end{definition}

\begin{definition}\label{def:partialsol}
  Let $\prog$ be a default theory, $\calT = (T, \chi, \delta)$ where
  $T=(N,\cdot,n)$ be an LTD of $S(\prog)$, and $t\in N$ be a node.  A
  \emph{partial solution for~$t$} is a tuple
  $(Z, {\cal M}, \PPP, {{\cal C}})$ where~
  $Z\subseteq\progt{t}$, and~$\PPP$ is the largest set of tuples
  such that each $(\sigma, \AAA, \BBB) \in \PPP$ is a
  partial extension under~$\gamma(Z)$ with~$\BF{\cal B\,} = \emptyset$
  and~$Z = \sigma^{-1}(\lgamma)$. Moreover, ${\cal C}$ is the largest
  set of tuples such that for each
  $(\rho, {\cal AC}, {\cal BC}) \in {{\cal C}}$, we have that
  $({\rho}, {\cal AC}, {\cal BC})$ is a partial extension
  under~$\gamma(Z)$
  with~$\rho^{-1}(\lgamma)\subsetneq \sigma^{-1}(\lgamma)$.  Finally,
  ${\cal M}\subseteq 2^{\atto}$ is the largest set
  with~$M\models\Gamma_t[\gamma(Z)]$ for each~$M\in{\cal M}$.
\end{definition}

\noindent The following lemma establishes correspondence between
stable default sets and partial solutions.

  \begin{lemma}\label{prop:partialsol_corr}
    Let $\prog$ be a default theory, $\calT = (T, \chi, \delta)$ be an LTD of
    the semi-primal graph~$S(\prog)$, where $T=(\cdot,\cdot,n)$,
    and $\chi(n) = \emptyset$.  Then, there exists a stable default
    set~$Z'$ for~$\prog$ if and only if there exists a partial
    solution $\tabval = (Z', {\cal M}, \PPP, {\cal C})$ for
    root~$n$ with at least one
    tuple~$\langle \sigma, \AAA, \BBB \rangle \in \PPP$
    where ${\cal B} = \emptyset$, 
    where~${\cal C}$ is of the following form: For each
    $(\rho, {\cal AC}, {\cal BC}) \in{\cal C}$, 
    ${\cal BC}_{\bfff}\neq{\cal BC}$. 
  \end{lemma}
\begin{proof}[Sketch]
  Given a stable default set ${Z'}$ of $\prog$ we construct
  $\tabval = (Z', {\cal M}, \PPP, {\cal C})$ where we generate every potential
  $\sigma: \prog \to  \{\lalpha,\lbeta,\lgamma\}$ such that
  $\sigma(d) = \lgamma$ for~$d\in Z'$ as follows.
  For~$d\in\prog\setminus Z'$, we are allowed to set $\sigma(d)\eqdef \lalpha$ if
  $\gamma(Z') \wedge \neg\alpha(d)$ is satisfiable and
  $\sigma(d)\eqdef\lbeta$ if~$\gamma(Z') \wedge \beta(d)$ is unsatisfiable.
  
  For each of this functions~$\sigma$, we
  require~$\langle \sigma, \AAA, \emptyset \rangle \in \PPP$,
  where
  $\AAA\subseteq 2^{\at(D)}$ is the smallest set with $\Card{\AAA} \leq
  \Card{\sigma^{-1}(\alpha)}-1$ such that for all
  $d\in\sigma^{-1}(\alpha)$ there is at least one $A\in \AAA$
  with $A\models \gamma(Z') \wedge \neg\alpha(d)$.
  
  Moreover, we define 
  set~${\cal M} \eqdef \Mod_{2^{\at(D)}}(\bigwedge_{d\in Z'} \gamma(d))$,
  in order for $\tabval$ to be a partial solution for $n$ (see
  Definition~\ref{def:partialsol}). We construct~${\cal C}$,
  consisting of partial solutions~$(\rho, {\cal AC}, {\cal BC})$ where
  we use every potential state function~$\rho$
  with~$\rho^{-1}(\lgamma) \subsetneq \sigma^{-1}(\lgamma)$.  For
  this, let~$Z\eqdef \rho^{-1}(\lgamma)$.  For the
  defaults~$d$ with~$\rho(d) \neq \lgamma$,~i.e., defaults~$d$ that are \asat
  or \bsat, we also set their
  state~$\rho(d)$ to~$\alpha$ or~$\beta$, respectively (analogous to
  above). 
  Finally, we define
  set~${\cal BC} \eqdef \MAIRCR{[\Mod_{2^{\at(D)}}(\bigwedge_{d\in Z}\gamma(d))}{\bff} \cup [\bigcup_{d:\rho(d)=\lbeta}{\Mod_{2^{\at(D)}}([\bigwedge_{d\in Z'}\gamma(d)] \wedge \beta(d))}]$,
  and~${\cal AC}$ $\subseteq 2^{\at(D)}$ as the smallest set such that $\Card{{\cal AC}} \leq
  \Card{\rho^{-1}(\lalpha)}-1$ and for all
  $d\in\rho^{-1}(\lalpha)$, there is at least one
  $AC\in{\cal AC}$ with $AC \models \gamma(Z) \wedge \neg\alpha(d)$ according to
  Definition~\ref{def:partial-model}.
  
  For the other direction, Definitions~\ref{def:partial-model} and
  \ref{def:partialsol} guarantee that $Z'$ is a stable extension if
  there exists such a partial solution~$\tabval$.
  In consequence, the lemma holds.
\end{proof}

\noindent Next, we require the notion of local partial solutions
corresponding to the tuples obtained in Algorithm~\ref{fig:prim}.

\begin{definition}\label{def:localpartialsolpart}
  Let $\prog$ be a default theory, $\calT = (T, \chi, \delta)$ an LTD of the
  semi-primal graph~$S(\prog)$, where $T=(N,\cdot,n)$, and $t\in N$
  be a node. A tuple $(\sigma, \AAA, \BBB)$ is a \emph{local
    partial solution part} of partial
  solution~$({\hat \sigma}, {\hat \AAA}, {\hat \BBB})$ for $t$
  if
  \begin{enumerate}
  \item
    $\sigma = {\hat \sigma} \cap (\chi(t) \times
    \{\lalpha, \lbeta, \lgamma\})$,
  \item $\AAA = {\hat \AAA}_t$, and
  \item $\BBB = {\hat \BBB}_t$, where
    ${\cal S}_t \eqdef \{S \cap (\chi(t) \cup \{\bff\}) \mid S
    \in{{\cal S}}\}$.
  \end{enumerate}
\end{definition}

\begin{definition}\label{def:localpartialsol}
  Let $\prog$ be a default theory, $\calT = (T, \chi, \delta)$ an LTD of the semi-primal
  graph~$S(\prog)$, where $T=(N,\cdot,n)$, and $t\in N$ be a node. A
  tuple $\tabval = \langle Z, {\cal M}, \PPP, {\cal C} \rangle$ is a
  \emph{local partial solution} for $t$ if there exists a partial
  solution
  ${\hat \tabval} = ({\hat Z}, {\hat {\cal M}}, {\hat \PPP}, {\hat {\cal C}})$
  for~$t$ such that the following conditions hold:
\begin{inparaenum}[(1)]
\item $Z = {\hat Z} \cap 2^{\prog_t}$,
\item ${\cal M} = {\hat {\cal M}}_t$,
\item $\PPP$ is the smallest set containing local partial solution part $(\sigma, \AAA, \BBB)$ for each~$({\hat \sigma}, {\hat \AAA}, {\hat \BBB})\in{\hat\PPP}$, and
\item ${\cal C}$ is the smallest set with local partial solution part $(\rho, {\cal AC}, {\cal BC})\in{\cal C}$ for each~$({\hat \rho}, {\hat {\cal AC}}, {\hat {\cal BC}})\in{\hat{\cal C}}$.

\end{inparaenum}
We denote by ${\hat \tabval}^t$ the local partial solution $\tabval$
for~$t$ given partial solution ${\hat \tabval}$.

\end{definition}

\noindent The following proposition provides justification that it
suffices to store local partial solutions instead of partial solutions
for a node~$t \in N$.

\begin{lemma}\label{prop:partiallocalsol_corr}
  Let $\prog$ be a default theory, $\calT = (T, \chi, \delta)$ an LTD of
  $S(\prog)$, where $T=(N,\cdot,n)$, and $\chi(n) = \emptyset$.  Then,
  there exists a stable default set set for $\prog$ if and only if there
  exists a local partial solution of the
  form~$\langle \emptyset, \{\emptyset\}, \PPP, {\cal C} \rangle$
  for the root~$n \in N$ with at least one tuple of the
  form~$\langle \sigma, \AAA, \emptyset\rangle\in\PPP$. 
  Moreover, for
  each~$\langle \rho, {\cal AC}, {\cal BC} \rangle$ in~${\cal C}$,
  $\BF{\cal BC\,} \neq {\cal BC}$. 
\end{lemma}
\begin{proof}
  Since $\chi(n) = \emptyset$, every partial solution for the root~$n$
  is an extension of the local partial solution~$\tabval$ for the
  root~$n \in N$ according to Definition~\ref{def:localpartialsol}. By
  Lemma~\ref{prop:partialsol_corr}, we obtain that the lemma is true.
\end{proof}

\noindent
In the following, we abbreviate variables occurring in bag~$\chi(t)$
by~$\at_t$,~i.e., $\at_t \eqdef \chi(t) \setminus \prog_t$.



  \begin{proposition}[Soundness]\label{thm:soundness}
    Let $\prog$ be a default theory, $\calT = (T, \chi, \delta)$ an LTD of the
    semi-primal graph~$S(\prog)$, where $T=(N,\cdot,\cdot)$, and
    $t\in N$ a node.  Given a local partial solution $\tabval'$ of
    child table $\tab{}'$ (or local partial solution $\tabval'$ of
    table $\tab{}'$ and local partial solution $\tabval''$ of table
    $\tab{}''$), each tuple $\tabval$ of table $\tab{t}$ constructed
    using table algorithm $\PRIM$ is also a local partial solution.
  \end{proposition}
\begin{proof}
  Let $\tabval'$ be a local partial solution for $t'\in N$ and
  $\tabval$ a tuple for node~$t \in N$ such that $\tabval$ was derived
  from~$\tabval'$ using table algorithm \PRIM. Hence, node~$t'$ is the
  only child of $t$ and~$t$ is either removal or introduce node.

  Assume that $t$ is a removal node and
  $d\in \prog_{t'}\setminus \prog_t$ for some default~$d$.
  %
  %
  Observe that for $\tabval = \langle Z, {\cal M}, \PPP, {\cal C} \rangle$ and
  $\tabval' = \langle Z', {\cal M}, {\cal P'}, {\cal C'} \rangle$, 
  sets~$\AAA$ and~${\cal B}$ are equal, i.e., $\langle \cdot, \AAA, {\cal B}\rangle \in\PPP \Longleftrightarrow \langle \cdot, \AAA, {\cal B}\rangle \in{\cal P'}$ and $\langle \cdot, \AAA, {\cal B}\rangle \in{\cal C} \Longleftrightarrow \langle \cdot, \AAA, {\cal B}\rangle \in{\cal C'}$.
  %
  Since $\tabval'$ is a local partial solution, there exists a partial
  solution~${\hat \tabval'}$ of $t'$, satisfying the conditions of
  Definition~\ref{def:localpartialsol}.  Then, ${\hat \tabval'}$ is
  also a partial solution for node $t$, since it satisfies all
  conditions of Definitions~\ref{def:partial-model} and
  \ref{def:partialsol}.  Finally, note that
  $\tabval = ({\hat {\tabval}'})^t$ since the projection of
  ${\hat \tabval'}$ to the bag $\chi(t)$ is $\tabval$ itself. In
  consequence, the tuple~$\tabval$ is a local partial solution.

  For~$a\in\at_{t'}\setminus\at_t$ as well as for introduce nodes, we
  can analogously check the proposition.

  Next, assume that $t$ is a join node. Therefore, let $\tabval'$ and
  $\tabval''$ be local partial solutions for $t',t''\in N$,
  respectively, and $\tabval$ be a tuple for node $t\in N$ such that
  $\tabval$ can be derived using both $\tabval'$ and $\tabval''$ in
  accordance with the \PRIM algorithm. Since $\tabval'$ and
  $\tabval''$ are local partial solutions, there exists partial
  solution
  ${\hat \tabval'} = ({\hat {Z'}}, {\hat {\cal M'}}, {\hat {\cal P'}}, {\hat {\cal C'}})$
  for node $t'$ and partial solution
  ${\hat \tabval''} = ({\hat {Z''}}, {\hat {\cal M''}}, {\hat {\cal P''}}, {\hat {\cal
      C''}})$ for node $t''$.  Using these two partial solutions, we
  can construct
  ${\hat \tabval} = ({\hat {Z'}} \cup {\hat {Z''}}, {\hat {\cal M'}} \bowtie {\hat {\cal M''}},
  {\hat {\cal P'}} \, {\hat\bowtie}_{ {\hat {\cal M'}}, {\hat {\cal M''}}}\, {\hat {\cal P''}}, ({\hat {\cal C'}} \, {\hat\bowtie}_{ {\hat {\cal M'}}, {\hat {\cal M''}}}\, {\hat {\cal
      C''}}) \cup ({\hat {\cal P'}}\, {\hat\bowtie}_{ {\hat {\cal M'}}, {\hat {\cal M''}}}\, {\hat {\cal C''}}) \cup ({\hat {\cal C'}} \, {\hat\bowtie}_{ {\hat {\cal M'}}, {\hat {\cal M''}}}\, {\hat {\cal P''}}))$ where
  for $\bowtie(\cdot, \cdot)$ and ${\hat \bowtie}(\cdot, \cdot)$ we refer to Listing~\ref{fig:prim}. 
  Then, we check all conditions of Definitions~\ref{def:partial-model}
  and \ref{def:partialsol} in order to verify that ${\hat
    \tabval}$ is a partial solution for
  $t$. Moreover, the projection ${\hat \tabval}^t$ of ${\hat
    \tabval}$ to the bag $\chi(t)$ is exactly
  $\tabval$ by construction and hence, $\tabval = {\hat
    \tabval}^t$ is a local partial solution.

  Since one can provide similar arguments for each node type, we established
  soundness in terms of the statement of the proposition.

\end{proof}

\begin{proposition}[Completeness]\label{prop:completeness}
  Let $\prog$ be a default theory, $\calT = (T, \chi, \delta)$ where
  $T=(N,\cdot,\cdot)$ be an LTD of $S(\prog)$ and $t\in N$ be a
  node. Given a local partial solution $\tabval$ of table $\tab{t}$,
  either $t$ is a leaf node, or there exists a local partial solution
  $\tabval'$ of child table $\tab{}'$ (or local partial solution
  $\tabval'$ of table $\tab{}'$ and local partial solution~$\tabval''$
  of table $\tab{}''$) such that $\tabval$ can be constructed by
  $\tabval'$ (or $\tabval'$ and $\tabval''$, respectively) and using
  table algorithm~${\PRIM}$.
\end{proposition}
\begin{proof}
  Let $t\in N$ be a removal node and
  $d\in \prog_{t'} \setminus \prog_t$ with child node~$t'\in N$.  We
  show that there exists a tuple~$\tabval'$ in table~$\tab{t'}$ for
  node $t'$ such that $\tabval$ can be constructed using $\tabval'$ by
  \PRIM (Listing~\ref{fig:prim}). Since $\tabval$ is a local
  partial solution, there exists a partial solution
  ${\hat \tabval} = ({\hat Z}, {\hat {\cal M}}, {\hat \PPP}, {\hat {\cal C}})$ for
  node~$t$, satisfying the conditions of
  Definition~\ref{def:localpartialsol}.  
  It is easy to see that ${\hat \tabval}$
  is also a partial solution for~$t'$ and we define
  $\tabval' \eqdef {\hat \tabval}^{t'}$, which is the projection of
  ${\hat \tabval}$ onto the bag of~$t'$. Apparently, the
  tuple~$\tabval'$ is a local partial solution for node $t'$ according
  to Definition~\ref{def:localpartialsol}. Then, $\tabval$ can be
  derived using \PRIM algorithm and $\tabval'$.  By similar arguments,
  we establish the proposition for~$a\in \at_{t'}\setminus\at_t$ and
  the remaining node types. Hence, the propositions sustains.
\end{proof}

\noindent
Now, we are in the situation to prove Theorem~\ref{thm:prim:correctness}, which
 states that we can decide the
problem~\Ext by means of Algorithm~$\dpa$.

\begin{restatetheorem}[thm:prim:correctness]
\begin{theorem}
  Given a default theory~$\prog$, the algorithm ${\dpa}$
  correctly solves \Ext.
\end{theorem}
\end{restatetheorem}

\begin{proof}
  We first show soundness. Let $\TTT = (T, \chi, \delta)$ be the given LTD,
  where $T = (N,\cdot ,n)$.
  By Lemma~\ref{prop:partiallocalsol_corr} we know that there is a
  stable default set if and only if there exists a local partial
  solution for the root~$n$. Note that the tuple is by construction of
  the
  form~$\langle\emptyfunc, \{\emptyset\}, \PPP, {\cal C} \rangle$,
  where~$\PPP\neq\emptyset$ can contain a combination of the
  following tuples~$\langle \emptyset, \emptyset, \emptyset\rangle$,
  $\langle \emptyset, \{\emptyset\}, \emptyset\rangle$.  For
  each~$\langle \rho, {\cal AC}, {\cal BC} \rangle \in {\cal C}$, we have
  $\BF{\cal BC\,}\neq {\cal BC}$.
  In
  total, this results in 16 possible tuples,
  since~${\cal C} \subseteq 2^C$ can contain any combination (4 many)
  of~$C$, where
  $C=\{\langle \emptyset, \emptyset, \{\emptyset, \{\bff\}\} \rangle$,
  $\langle \emptyset, \{\emptyset\}, \{\emptyset,\{\bff\}\}
  \rangle\}$.
    Hence, we proceed by induction starting from the leaf
  nodes in order to end up with such a tuple at the root node~$n$. In fact, the
  tuple~$\langle \emptyset, \{\emptyset\}, \{\langle \emptyset, \emptyset, \emptyset \rangle\}, \emptyset \rangle$ is trivially
  a partial
  solution for (empty) leaf nodes 
  by Definitions~\ref{def:partial-model} and~\ref{def:partialsol} and
  also a local partial solution of
  $\langle \emptyset, \{\emptyset\}, \{\langle \emptyset, \emptyset, \emptyset \rangle \}, \emptyset \rangle$ by
  Definition~\ref{def:localpartialsol}.  We already established the
  induction step in Proposition~\ref{thm:soundness}.
  Hence, when we reach the root~$n$, when traversing the TD in
  post-order by Algorithm~$\dpa$, we obtain only valid tuples
  inbetween and a tuple of the
  form discussed above 
  in the table of the root~$n$ witnesses an answer set.
  
  Next, we establish completeness by induction starting from the
  root~$n$. Let therefore, ${\hat Z}$ be an arbitrary stable default set of~$\prog$.
  By Lemma~\ref{prop:partiallocalsol_corr}, we know that for the root~$n$ there exists a local partial solution of
  the discussed form~$\langle \emptyset, \{\emptyset\}, \PPP, {\cal C} \rangle$ for some partial
  solution~$\langle {\hat Z}, {\hat {\cal M}}, {\hat \PPP}, {\hat{\cal C}} \rangle$. 
  We already established the induction step in
  Proposition~\ref{prop:completeness}. 
  Hence, we obtain some (corresponding) tuples for every
  node~$t$. Finally, stopping at the leaves~$n$.
  In consequence, we have shown both soundness and completeness
  resulting in the fact that Theorem~\ref{thm:prim:correctness} is
  true.
\end{proof}

\begin{proposition}[Completeness for Enumeration]\label{prop:completeness2}
  Let $\prog$ be a default theory, $\calT = (T, \chi, \delta)$ where
  $T=(N,\cdot,\cdot)$ be an LTD of $S(\prog)$ and $t\in N$ be a
  node. Given a partial solution $\hat\tabval$ and the corresponding
  local partial solution~$\tabval = {\hat\tabval}^t$ for table $\tab{t}$,
  either $t$ is a leaf node, or there exists a local partial solution
  $\tabval'$ of child table $\tab{}'$ (or local partial solution
  $\tabval'$ of table $\tab{}'$ and local partial solution~$\tabval''$
  of table $\tab{}''$) such that $\tabval$ can be constructed by
  $\tabval'$ (or $\tabval'$ and $\tabval''$, respectively) and using
  table algorithm~${\PRIM}$.
\end{proposition}
\begin{proof}[Idea]
The correctness proof requires to extend the previous
results to establish a one-to-one correspondence when traversing the
tree of the TD and such that we can reconstruct each solution as well as we do not
get duplicates. The result then follows from the proof for completeness (see Proposition~\ref{prop:completeness}).
\end{proof}

\begin{restatetheorem}[thm:enumsafe]
\begin{theorem}
  Given a default theory~$\prog$, the algorithm~$\dpa$ can be
  used as a preprocessing step to construct tables from which we can
  correctly solve the problem~\AspEnum. More precisely, this is solved by first
  running Algorithm~$\dpa$, constructing the $\prec$-smallest solution~${\cal S}$, and then running
  Algorithm~${\nxt}_{\prec}({\cal T}, {\cal S})$ on the resulting tables of
  Algorithm~$\dpa$ until ${\nxt}_{\prec}({\cal T}, {\cal S})$ returns
  ``undefined''.
\end{theorem}
\end{restatetheorem}

For showing the theorem, we require the following three results.

\begin{observation}\label{lem:funcdep}
Let $\prog$ be a default theory, $\calT = (T, \chi, \delta)$ where
  $T=(N,\cdot,\cdot)$ be an LTD of $S(\prog)$ and $t\in N$ be a
  node. Then, for each partial solution~$\tabval=\langle Z, {\cal M}, \PPP, {\cal C} \rangle$ for~$t$,
${\cal M}, \PPP$ and ${\cal C}$ are functional dependent from~$Z$,
i.e., for any partial solution~$\tabval' = \langle Z, {\cal M'}, {\cal P'}, {\cal C'} \rangle$ for~$t$,
we have $\tabval=\tabval'$. 
\end{observation}
\begin{proof}
The claim immediately follows from Definition~\ref{def:partialsol}.
\end{proof}

\begin{algorithm}[t]
  \KwData{TD~$\TTT=(T,\cdot,\cdot)$ with
    $T=(N,\cdot,n)$, solution tuples~${\cal S}$, total ordering~$\prec$ of $\orig_\cdot(\cdot)$.}
  \KwResult{The next solution tuples of~${\cal S}$ using~$\prec$. } %
  $\Tabs{$\cdot$} \lassign \dpa({\cal T})$
  
  \For{\text{\normalfont iterate} $t$ in \text{\normalfont post-order}(T,n)}{\vspace{-0.05em}%
    $\Tab{} \eqdef \SB \Tabs{$t'$} \SM t' \text{ is a child of $t$ in $T$}\SE$\;\vspace{-0.05em} %
	
	${\hat t} \eqdef \text{\normalfont parent of } t$
	
	${\cal S}[t] \lassign \text{\normalfont direct successor } s' \succ {\cal S}[{t}] \text{ in } \orig_{{\hat t}}({\cal S}[{\hat t}])$
	
	\If{${\cal S}[t]$ defined}{
		\For{\text{\normalfont iterate} $t'$ in \text{\normalfont \Tab{}}}{
			\For{\text{\normalfont iterate} $t''$ in \text{\normalfont pre-order}(T,t')}{\vspace{-0.05em}%
				${\hat t''} \eqdef \text{\normalfont parent of } t''$
				
				${\cal S}[t''] \lassign$ \text{\normalfont $\prec$-smallest element in $\orig_{{\hat t''}}({\cal S}[{\hat t''}])$} 
			\vspace{-0.3em}}
		\vspace{-0.3em}}
		\Return{${\cal S}$}\;
	\vspace{-0.3em}}
    \vspace{-0.3em} }\vspace{-0.1em}%
    \Return{undefined}\;
  \caption{Algorithm ${\nxt}_{\prec}({\cal T}, {\cal S})$ for computing the next stable default set of~${\cal S}$.}
\label{fig:dpnext}
\end{algorithm}
%

\begin{lemma}\label{lem:charset}
Let $\prog$ be a default theory, $\calT = (T, \chi, \delta)$ with 
  $T=(N,\cdot,\cdot)$ be an LTD of $S(\prog)$, 
  and~$Z$ be a  stable default set.
  Then, there is a unique
set of tuples~$S$, containing exactly one tuple per node~$t\in N$ 
containing only local partial solutions of 
the unique partial solution for~$Z$. 
\end{lemma}
\begin{proof}
  By Observation~\ref{lem:funcdep}, given~$Z$, we can construct one
  unique partial
  solution~${\hat \tabval}=\langle Z, {\cal M}, \PPP, {\cal C}
  \rangle$ for~$n$. We then define the set~$S$
  by~$S\eqdef \bigcup_{t\in N} \{{\hat\tabval}^t\}$.  Assume that
  there is a different set~$S'\neq S$ containing also exactly one
  tuple per node~$t\in N$. Then there is at least one node~$t\in N$,
  for which the corresponding tuples~$\tabval\in S, \tabval' \in S'$
  differ ($\tabval\neq\tabval'$), since~${\hat \tabval}$ is unique and
  the computation~${\hat \tabval}^{t}$ is defined in a deterministic,
  functional way (see Definition~\ref{def:localpartialsol}).  Hence,
  either ${\hat \tabval}^t \neq \tabval$ or
  ${\hat \tabval}^t \neq \tabval'$, leading to the claim.
\end{proof}

\begin{proposition}\label{thm:lineardelay}
  Let $\prog$ be a default theory, $\calT = (T, \chi, \delta)$
  with 
  $T=(N,\cdot,\cdot)$ be an LTD of $S(\prog)$, 
  and~$Z$ be a stable default set.  Moreover, let~$S$
  be the unique set of tuples, containing exactly one tuple per
  node~$t\in N$ and containing only local partial solutions of the
  unique partial solution for~$Z$. Given~$S$, and tables of
  Algorithm~$\PRIM$, one can compute in
  time~${\cal{O}}(\CCard{\prog})$ a stable
  default set~$Z'$ with~$Z'\neq Z$, assuming one can get for a specific
  tuple~$\tabval$ for node~$t$ its corresponding ~$\prec$-ordered
  predecessor tuple set~$\orig_t(\tabval)$ of tuples in the child
  node(s) of~$t$ in constant time.
\end{proposition}
\begin{proof}
Note that with~$Z$, it is easy to determine, which element of~$S$
belongs to which node~$t$ in~$T$, hence, we 
can construct a mapping~${\cal S}: N \rightarrow S$.
With~${\cal S}$, we can easily apply algorithm~\nxt, which is given in Listing~\ref{fig:dpnext},
in order to construct a different solution~${\cal S'}$ in a systematic way
with linear time delay, since~${\cal T}$ is nice.
\end{proof}

\begin{proof}[of Theorem~\ref{thm:enumsafe} (Sketch)]
First, we construct an LTD~${\TTT}=(T,\chi, \delta)$ with $T=(N,n)$
for graph~$S(\prog)$. Then we run our algorithm~${\dpa}$
and get tables for each TD node.
In order to enumerate all the stable default sets,
we investigate each of these tuple, which lead to a valid stable
default set (see proof of Theorem~\ref{thm:prim:correctness}).
For each of these tuples (if exist), we construct a first solution~$S$,
if exist, (as done in Lines~7 to~10 of Listing~\ref{fig:dpnext}, for the root~$n$)
using~$\orig_t(\cdot)$, and total order~$\prec$.  Thereby, we
keep track of which tuple in~$S$ belongs to which node, resulting in
the mapping~${\cal S}$ (see proof of
Proposition~\ref{thm:lineardelay}).  Note that~$\orig_t(\cdot)$
and~$\prec$ can easily be provided by remembering for each tuple an
ordered set of predecessor tuple sets during construction (using table
algorithm~$\PRIM$).  Now, we call
algorithm~$\nxt_{\prec}({\cal T}, {\cal S})$ multiple times, by outputting
and passing the result again as argument, until the return value is
undefined, enumerating solutions in a systematic way.  Using
correctness results (by Theorem~\ref{thm:prim:correctness}), and
completeness result for enumeration by
Proposition~\ref{prop:completeness2}, we obtain only valid solution
sets, which directly represent stable default sets and,
in particular, we do not miss a single one.  Observe, that we do not
get duplicates (see Lemma~\ref{lem:charset}).
\end{proof}

\section{Proof of Runtime Guarantees}

\begin{restatetheorem}[thm:prim:runtime]
\begin{theorem}
  Given a default theory~$\prog$, the algorithm ${\dpa}$ and
  runs in time $\bigO{{2^{2^{2^{k + 4}}}}\cdot \CCard{S(\prog)}}$, where
  $k\eqdef\tw{S(\prog)}$ is the treewidth of the semi-primal
  graph~$S(\prog)$.
\end{theorem}
\end{restatetheorem}

First, we give a proposition on worst-case space requirements in
tables for the nodes of our algorithm.

\begin{proposition}\label{prop:prim:space}
 Given a default theory \prog, an LTD ${\cal T} = (T, \chi, \delta)$ with
  $T=(N, \cdot, \cdot)$ of the semi-primal graph~$S(\prog)$, and a node
  $t\in N$. Then, there are at most
  $2^{k+1}\cdot2^{2^{k+1}}\cdot2^{2\cdot(3^{k+1}\cdot2^{2^{k+2}})}$
  tuples in $\tab{t}$ using algorithm ${\dpa}$ for width $k$ of
  ${\cal T}$.
\end{proposition}
\begin{proof}[Sketch] %
  Let $\prog$ be the given default theory, ${\cal T} = (T, \chi, \delta)$ an LTD of the
  semi-primal graph~$S(\prog)$, where $T=(N, \cdot, \cdot)$, and
  $t\in N$ a node of the TD. Then, by definition of a decomposition of
  the semi-primal graph for each node~$t\in N$, we
  have~$\Card{\chi(t)} - 1 \leq k$.  In consequence, we can have at
  most $2^{k+1}$ many witness defaults and
  $2^{2^{k+1}}$ many witnesses models.
  Each set~$\PPP$ may contain a set of witness proof tuples of the
  form~$\langle \sigma, \AAA, \BBB \rangle$,
  with at most $3^{k+1}$ many witness state~$\sigma$ mappings,
  $2^{2^{k+1}}$ many backfire witness models~$\BBB$,
  and ${2^{2^{k+1}}}$ many required witnesses model sets.
  In the end, we need to distinguish
  $2^{k+1}\cdot2^{2^{k+1}}\cdot 2^{(3^{k+1}\cdot{2^{2^{k+2}}})}$
  different witnesses of a tuple in
  the table~$\tab{t}$ for node~$t$.
  For each witness, we can have at most
  $2^{(3^{k+1}\cdot{{2^{2^{k+2}}})}}$ many counter-witnesses per
  witness default, witness models, and required witness model sets. Therefore, there are at most
  $2^{k+1}\cdot{2^{2^{k+1}}}\cdot2^{2\cdot(3^{k+1}\cdot{2^{2^{k+2}}})}$ tuples
  in table~$\tab{t}$ for node~$t$.
  In consequence, we established the
  proposition.
\end{proof}

\begin{proof}[of Theorem~\ref{thm:prim:runtime}]
  Let $\prog$ be a default theory, $S(\prog)=(V,\cdot)$ its semi-primal graph, and
  $k$ be the treewidth of $S(\prog)$.
  Then, we can compute in time~$2^{\bigO{k^3}} \cdot \Card{V}$ an LTD
  of width at
  most~$k$~\cite{BodlaenderKoster08}. 
  We take such a TD and compute in linear time a nice
  TD~\cite{BodlaenderKoster08}. 
  Let $\TTT = (T,\chi,\delta)$ be such a pretty LTD with $T = (N,\cdot,\cdot)$.
  Since the number of nodes in~$N$ is linear in the graph size and
  since for every node~$t \in N$ the table~$\tab{t}$ is bounded by
  $2^{k+1}\cdot2^{2^{k+1}}\cdot2^{2\cdot(3^{k+1}\cdot{2^{2^{k+2}}})}$
  according to Proposition~\ref{prop:prim:space}, we obtain a running
  time of~$\bigO{{{2^{2^{2^{k+4}}}}}\cdot\CCard{S(\prog)}}$. Consequently, the theorem sustains.
\end{proof}
}

\trash{
\johannes{TRASH: REMOVE}

\todo{danger zone}
The algorithm \PRIM operate in several cases and for each bag we compute a separate table containing sets of all possible tuples of satisfying defaults, additional information to identify the satisfiability of formulas, and counter-candidates that
contradict to extend the satisfying defaults to stable satisfying defaults.
In the first case we introduce stepwise the variables and the defaults. \\
While we introduce a variable, we guess the value of this variable and while we introduce a default we simply guess, if this default is \asat then we increase the set of generating default, or not because of $\alpha$ or $\beta$ and store the according guesses in each row of the table.  \\
Based on the design of the semi-primal graph, all elements of the default are connected as a clique and in this way have to occur in one bag. As soon as we reach such a bag containing all elements of the default we start the labeling cases.\\
The case concl labeling operate on the row of the table with the guess of firing default, thereby we update all possible models in $\cal M$ and we need to ensure that the so-called \textit{require witness set} $\cal R$ contains models of $\gamma(d)$ only.\\
In the case of labeling of prerequisite  we operate in the row, of not firing default due to $\alpha(d)$, and we have to guarantee $\cal R$ contains all possible set of models for $\neg \alpha(d)$. \\
And the label of justification examine the set of models in $\cal M$ to be also a model of $\beta(d)$, in this case we mark the model with so-called \textit{backfire} variable $\bff$.\\
Additionally to introduce cases we also have remove cases. So by removing a variable, we remove all occurrences of the variable within the tuples. In a similar vein we remove a default and appropriate mapping $\sigma$ to $\lgamma$, $\lalpha$ and $\lbeta$.\\
And finally the last case, is the join case, where we have to merge the results of both branches according to certain rules to compute a main result and pass it to the root. 
\todo{danger zone}

\begin{figure}
 \centering
  \begin{tikzpicture}[every node/.style={font=\sffamily}, state/.style={circle,fill=black,minimum width=2mm,inner sep=0mm},y=1.3cm,x=1.5cm]
 \node[state,label={[yshift=0em] above:\footnotesize$D_1\eqdef \left\{ d_{1}=\frac{\top : a}{a \vee b},
      d_{2}=\frac{\top :\neg a}{\neg b}\right\}$}] (D) at (0,3) {};

 \node[state,label={left:\footnotesize repr},label={[xshift=0.3em]right:\footnotesize def$_1$}] (def1) at (-1,2.25) {};
 \node[state,label={[xshift=-0.3em,yshift=-0.1em]left:\footnotesize repr},label={right:\footnotesize def$_2$}] (def2) at (1,2.25) {};
    \node[state,label={left:\footnotesize repr},label={[xshift=0.3em]right:\footnotesize $\alpha_1$}] (pre1) at (-2.5,1) {};
   \node[state,label={left:\footnotesize repr},label={right:\footnotesize $\beta_1$}] (just1) at (-1.5,1) {};
   \node[state,label={left:\footnotesize repr},label={right:\footnotesize $\gamma_1$}] (con1) at (-.5,1) {};
\node[state,label={left:\footnotesize repr},label={right:\footnotesize $\alpha_2$}] (pre2) at (.5,1) {};
   \node[state,label={left:\footnotesize repr},label={right:\footnotesize $\beta_1$}] (just2) at (1.5,1) {};
   \node[state,label={left:\footnotesize repr},label={right:\footnotesize $\gamma_2$}] (con2) at (2.5,1) {};
%
\node[state,label={left:\footnotesize repr},label={[yshift=0.1em]right:\footnotesize $ \neg a$}] (na) at (1.2,0.2) {};
\node[state,label={left:\footnotesize repr},label={[yshift=0.1em]right:\footnotesize $\neg b$}] (nb) at (2.5,0) {};

The property treewidth was originally introduced for graphs and is
based on the concept of a tree decomposition, which arranges a tree
where each node consists of sets of the original graph vertices (bags)
such that certain conditions hold.

\node[state,label={below left:\footnotesize var $\top$}] (top) at (-0.5,-0.5) {};
\node[state,label={below:\footnotesize var $a$}] (a) at (0,-0.5) {};
\node[state,label={below right:\footnotesize var $b$}] (b) at (0.5,-0.5) {};

\path[-stealth'] (D) edge node [midway, above, sloped,outer sep=-1mm] {\footnotesize body$_{\text{def1}}$} (def1);
\path[-stealth'] (D) edge node [midway, above, sloped,outer sep=-1mm] {\footnotesize body$_{\text{def2}}$} (def2);
\path[-stealth'] (def1) edge node [midway, above, sloped,outer sep=-1mm] {\footnotesize body$_{\text{pre}_1}$} (pre1);
\path[-stealth'] (def1) edge node [midway, above, sloped,outer sep=-1mm] {\footnotesize body$_{\text{just}_1}$} (just1);
\path[-stealth'] (def1) edge node [midway, above, sloped,outer sep=-1mm] {\footnotesize body$_{\text{conc}_1}$} (con1);
\path[-stealth'] (def2) edge node [midway, above, sloped,outer sep=-1mm] {\footnotesize body$_{\text{pre}_2}$} (pre2);
\path[-stealth'] (def2) edge node [midway, above, sloped,outer sep=-1mm] {\footnotesize body$_{\text{just}_2}$} (just2);
\path[-stealth'] (def2) edge node [midway, above, sloped,outer sep=-1mm] {\footnotesize body$_{\text{conc}_2}$} (con2);
\path[-stealth'] (pre1) edge node [midway, above, sloped,outer sep=-1mm] {\footnotesize body$_{\text{var}(\top)}$} (top);
\path[-stealth'] (just1) edge node [midway, above, sloped,outer sep=-1mm] {\footnotesize body$_{\text{var}(a)}$} (a);
\path[-stealth'] (con1) edge node [midway, above, sloped,outer sep=-1mm] {\footnotesize body$_{\text{var}(a)}$} (a);
\path[-stealth'] (con1) edge node [midway, above, sloped,outer sep=-1mm] {\footnotesize body$_{\text{var}(b)}$} (b);
\path[-stealth'] (just2) edge node [midway, above, sloped,outer sep=-1mm] {\footnotesize body$_{\neg}$} (na);
\path[-stealth'] (con2) edge node [midway, above, sloped,outer sep=-1mm] {\footnotesize body$_{\neg}$} (nb);
\path[-stealth'] (na) edge node [midway, above, sloped,outer sep=-1mm] {\footnotesize body$_{\text{var}(a)}$} (a);
\path[-stealth'] (nb) edge node [midway, above, sloped,outer sep=-1mm] {\footnotesize body$_{\text{var}(b)}$} (b);

\draw (pre2) to (top);

\end{tikzpicture}
 \caption{}\label{fig:Aphi}
\end{figure}

%

\todo{BEGIN: trash}

for each default~$d\in D$ it holds that $E \cup \neg\alpha(d)$ is
satisfiable or $E \cup \beta(d)$ is unsatisfiable or
$\gamma(d) \in E$.

 and hence we need to do more in
the future

$Z'$ is a set of satisfying defaults of each
default~$d \in D_{\leq t}$, but there is some~$d \in D_t$ such that
$Z'$ is \emph{not} a set of satisfying defaults .s

$Z'$ does not satisfy $D_{\leq t} \cup \{d\}$.

$Z' \supseteq Z$ of $D_{\leq t}$

The idea is that satisfying set of defaults~$Z'\supseteq Z$ of
$\progtneq{t}$, but not of $\progtneq{t} \cup \{d\}$

set of defaults~$Z'\supseteq Z$, where~$d$ ``backfires'' in
default-theory~$\progtneq{t} \cup \{d\}$, i.e., $d$ certainly is not
\bsat.

modele die beta widersprechen

$d$ is not yet satisfied.

The set~$\BBB$, which we call the \emph{refuting \lbeta\hy
  assignments},

containing variable~$\bff$.  Such a model~$B\in\BBB$ cannot only be
extended to a model of~$F'$, but also to a model
of~$F'\wedge \beta(d)$ for some
default~$d\in\sigma^{-1}(\lbeta)\cup\progtneq{t}$ and indicates that
there is a set of defaults~$Z'\supseteq Z$, where~$d$ ``backfires'' in
default-theory~$\progtneq{t} \cup \{d\}$, i.e., $d$ certainly is not
\bsat.  

The set~$\BBB$, consists of \emph{backfire witness models},

To be more concrete, we additionally have to find at least one
model~$R'$   Since
such a model~$R'$ is required for
each~$d\in\progtneq{t} \setminus Z'$, which is
\asat,

contains a potential model~ for each such
default~$d$.

For each

Note that~$\AAA$ is kept as long as each~$R\in\AAA$ can be extended
to~$R'\supseteq R$ which satisfies~$F'$, and otherwise the whole set
is excluded, resulting in the consequence that the whole
tuple~$\langle \sigma, \BBB, \AAA\rangle$ vanishes from~$\PPP$.

For assignment~$\sigma(d)=\lalpha$ to be correct, it is required
that~$d$ is \asat, i.e., $\alpha(d)$ shall
not be a consequence of~$F'$.

The set~$\BBB$, which we call the \emph{refuting \lbeta\hy
  assignments},

containing variable~$\bff$.  Such a model~$B\in\BBB$ cannot only be
extended to a model of~$F'$, but also to a model
of~$F'\wedge \beta(d)$ for some
default~$d\in\sigma^{-1}(\lbeta)\cup\progtneq{t}$ and indicates that
there is a set of defaults~$Z'\supseteq Z$, where~$d$ ``backfires'' in
default-theory~$\progtneq{t} \cup \{d\}$, i.e., $d$ certainly is not
\bsat.  
In the end, if $Z$ proves the existence of a satisfying default
set~$Z'\supseteq Z$ of theory $\progtneq{t}$, then there is at least
one tuple~$\langle \cdot, \BBB, \cdot \rangle \in \PPP$ with
$\BBB = \emptyset$. Hence, we require that~$\BBB = \emptyset$
in order to guarantee that \emph{none} of these
defaults~$d\in\progtneq{t}$ backfires (no need to differentiate
``locally'').

for each default~$d\in D$ it holds that $E \cup \neg\alpha(d)$ is
satisfiable or $E \cup \beta(d)$ is unsatisfiable or
$\gamma(d) \in E$.

\todo{tuple?}

\todo{TRASH}
In
the end, we desire to have proof~$S=\progtneq{n}$ for the empty TD root~$n$. If there is
such a
tuple~$\langle \sigma, \{\{\bff\}\}, \cdot \rangle \in \PPP$, this
indicates a ``wrong'' proof, which assigned some default~$d$ in some TD node below to~$\lbeta$ 
in the wrong way.

Consequently, $S = \emptyset$ if there is no
tuple~$\langle \sigma, \cdot, \AAA\rangle\in\PPP$ with
either~$\sigma^{-1}(\lalpha)=\emptyset$
or~$\AAA\neq\emptyset$,~i.e., if at least one default~$d\in\progtneq{t}$ is \asat,
we require to have at least one~$R\in\AAA$ in order
to guarantee that \emph{all} of the defaults are satisfied. \todo{}  Finally,
we end up with~$S=\prog_t$ if there is at least one
tuple~$\langle \sigma, \BBB, \AAA\rangle\in\PPP$
with~$\BBB = \emptyset$ and
either~$\sigma^{-1}(\lalpha)=\emptyset$ or $\AAA\neq\emptyset$.
Otherwise, $S=\emptyset$.

\todo{trash \#2}

If such a default~$d$ is \bsat, we can
conclude that for any satisfying set~$E$ of defaults, we have
$E \wedge \beta(d) \models \bot$.  In order to check this property, we
construct ``backfire assignments'', since these assignments proof that
our default so to say ``backfires'', meaning our decision was
wrong. For this purpose, we store for each state function~$\sigma$
(trying to proof~$Z$) a set~$\BBB$ of these (partial) backfire
assignments (tagged by setting variable~$\bff$ to true) and,
intuitively, propagate this information from the corresponding TD node
to the TD root.  If there is still such an assignment left at the root
node for a state function~$\sigma$, this function can never proof
that~$Z$ marks the existence of a satisfying set of defaults.
Further, if~$d$ is \asat, we
have~$E \wedge \neg\alpha(d) \not\models \bot$.  In order to validate
this statement, we require (and have to guarantee) the existence of at
least one model of~$E$ and $\neg\alpha(d)$.

\todo{END: trash}

\todo{BEGIN: Markus Notes}

1) we have to satisfy each default
How can we do that:
make alpha, beta, or gamma satisfiable

Intuitively, we guess a set~$Z\subseteq D$ of defaults such that for
each bag-default~$d$ we satisfy Property~(i) whether the conclusions of~$Z$ satisfy

we guess the potential cause of the default for being satisfied
and remember the decisions in a (set of) ``states functions''~$\sigma$
for each default of the current bag.

This can be due to the precondition~$\alpha(d)$ or the
justification~$\beta(d)$ of~$d$. In the end, we store a set~$\PPP$
(potential ``proofs'') of these state functions~$\sigma$ for each
set~$Z$, with the goal that at least one of these state functions
``proves''~$Z$.

checking via proofs

So, eventually we have to ensure that at least one of these cause
combinations was correct, since otherwise we dissatisfy at least one
default and fail to proof~$Z$.

We do this by means of ``required assignments''.  Now, since such a
model has to exist for every default~$d$, which is \asat, 
we have to remember such a required assignment~$R$
for each of these cases, resulting in a set~$\AAA$ of these
assignments for each state function~$\sigma$.  Each of the
sets~$\AAA$ marks so to say one configuration, containing for each
default~$d$, which is \bsat, (successor)
parts of required assignments restricted to the current bag. The main
crux here is, that we have to enforce that none of these
assignments~$R\in \AAA$ gets kicked out. Or in other words, if
this is the case, the combination of state function~$\sigma$, backfire
assignments~$\BBB$ and required assignments~$\AAA$ vanishes
from the set~$\PPP$. So, there can be different
sets~$\BBB, \AAA$ assigned to~$\sigma$, i.e.,
$\langle \sigma, \BBB, \AAA \rangle$ forms a potential proof
in~$\PPP$.  In the end, when we reach the root node, if we chose
at least one default to be \bsat, we have to
enforce $\AAA\neq\emptyset$
for~$\langle\sigma,\BBB,\AAA\rangle$ to mark the existence of
a satisfying set of defaults.  Finally, in order to compute even
satisfying generating defaults, in each TD node, we try to invalidate
the subset-optimality of set~$Z$, i.e., conceptually, we do the same
as above, but again for each subset in a set of
counter-witnesses~${\cal C}$.


\todo{END: Markus Notes}
}

\end{document}